\documentclass[letterpaper]{article} 
\usepackage{aaai2026}  
\usepackage{times}  
\usepackage{helvet}  
\usepackage{courier}  
\usepackage[hyphens]{url}  
\usepackage{graphicx} 
\usepackage{natbib}  
\usepackage{caption} 
\usepackage{algorithm}
\usepackage{algorithmic}
\usepackage{multirow}
\usepackage{makecell}

\usepackage{amsmath}\allowdisplaybreaks
\usepackage{amsfonts,bm}
\usepackage{amssymb}
\usepackage{amsthm}
\usepackage{algorithm}
\usepackage{algorithmic}

\theoremstyle{plain}
\newtheorem{theorem}{Theorem}

\newtheorem{lemma}{Lemma}

\newtheorem{remark}{Remark}
\newtheorem{corollary}{Corollary}

\def\cA{{\mathcal{A}}}
\def\cB{{\mathcal{B}}}
\def\cC{{\mathcal{C}}}

\def\cE{{\mathcal{E}}}
\def\cF{{\mathcal{F}}}

\def\cL{{\mathcal{L}}}

\def\cN{{\mathcal{N}}}

\def\cZ{{\mathcal{Z}}}

\def\BE{{\mathbb{E}}}

\def\BI{{\mathbb{I}}}

\newcommand{\Var}{\mathrm{Var}}

\DeclareMathOperator*{\argmax}{arg\,max}
\DeclareMathOperator*{\argmin}{arg\,min}

\usepackage{newfloat}
\usepackage{listings}
\DeclareCaptionStyle{ruled}{labelfont=normalfont,labelsep=colon,strut=off} 
\floatstyle{ruled}
\newfloat{listing}{tb}{lst}{}
\floatname{listing}{Listing}
\title{Robust Decentralized Multi-armed Bandits:\\ From Corruption-Resilience to Byzantine-Resilience}
\author{
    Zicheng Hu, Yuchen Wang, Cheng Chen\thanks{The corresponding author.} \\
}
\affiliations{
    MoE Engineering Research Center of Hardware/Software
    Co-design Technology and Application \\
    East China Normal University \\
    \{51275902019, 51285902003\}@stu.ecnu.edu.cn, chchen@sei.ecnu.edu.cn
}

\usepackage{bibentry}

\begin{document}

\maketitle

\begin{abstract}
    Decentralized cooperative multi-agent multi-armed bandits (DeCMA2B) considers how multiple agents collaborate in a decentralized multi-armed bandit setting. Though this problem has been extensively studied in previous work, most existing methods remain susceptible to various adversarial attacks. In this paper, we first study DeCMA2B with adversarial corruption, where an adversary can corrupt reward observations of all agents with a limited corruption budget. We propose a robust algorithm, called DeMABAR, which ensures that each agent’s individual regret suffers only an additive term proportional to the corruption budget. Then we consider a more realistic scenario where the adversary can only attack a small number of agents. Our theoretical analysis shows that the DeMABAR algorithm can also almost completely eliminate the influence of adversarial attacks and is inherently robust in the Byzantine setting, where an unknown fraction of the agents can be Byzantine, i.e., may arbitrarily select arms and communicate wrong information. We also conduct numerical experiments to illustrate the robustness and effectiveness of the proposed method.
\end{abstract}


\section{Introduction}
\label{sec:introduction}
The multi-armed bandit (MAB) problem is a classical online learning model. It has been widely applied in many real-world scenarios such as wireless monitoring~\cite{le2014sequential}, medical trials~\cite{villar2015multi}, and online advertising~\cite{schwartz2017customer}. In a typical MAB setup, there are $K$ arms, each with an unknown reward distribution.

With advancements in social networks, data centers, and communication devices, multi-agent MAB problems have gained considerable attention~\cite{boursier2019sic,chawla2020gossiping,huang2021federated,liu2021cooperative,wang2022achieving,wang2022distributed,zuo2023adversarial}. Among the diverse multi-agent settings, we focus on the decentralized cooperative multi-agent multi-armed bandits (DeCMA2B), where $V$ agents collaborate on a network, sharing information with their neighbors. Many recent studies focus on improving the communication efficiency and achieving near-optimal regret bounds~\cite{martinez2019decentralized,lalitha2021bayesian,zhu2021decentralized,zhu2023byzantine}. However, few of them consider the robustness of the algorithms.

In real-world applications, multi-agent systems could be disrupted by many factors, such as click fraud~\cite{lykouris2018stochastic}, denial-of-service (DoS) attacks in routing~\cite{zhou2019privacy}, and the presence of malicious agents~\cite{ferdowsi2019cyber}. These are mainly studied under two regimes: (i) \textit{adversarial corruptions}~\cite{liu2021cooperative,ghaffari2024multi,hu2025near}, where an adversary can maliciously corrupt the rewards of an unknown proportion $\beta \in [0,1]$ of the $V$ agents with a total corruption level $C$, and (ii) \textit{Byzantine agents}~\cite{vial2022robust,zhu2023byzantine}, where compromised agents can behave arbitrarily and send conflicting information to neighbors. Thus, a natural and important question arises:

\textit{Is there a robust algorithm that can defend against both adversarial corruptions and Byzantine agents?}

\begin{table*}[t]
    \centering
    \renewcommand\arraystretch{2}
    \begin{tabular}{|c|c|c|c|}
    \hline
        \multicolumn{2}{|c|}{Methods} & Centralized & Decentralized \\
        \cline{1-4}
        \multicolumn{2}{|c|}{\cite{liu2021cooperative}} & $VC + \frac{K\ln^2(T)}{\Delta}$ & -- \\
         \cline{1-4}

         \multicolumn{2}{|c|}{\cite{ghaffari2024multi}} & $\frac{C}{V} + \frac{K\ln^2(T)}{V\Delta}$ & -- \\
         \cline{1-4}

         \multicolumn{2}{|c|}{\cite{hu2025near}} & $\frac{C}{V} + \sum_{\Delta_k>0}\frac{\ln^2(T)}{V\Delta_k} + \frac{K}{V\Delta}$ & -- \\
         \cline{1-4}
          \multirow{2}{*}[-0.5em]{\makecell{DeMABAR (Ours)}} &
          $\beta \leq \alpha$  & $\frac{1}{1-2\alpha}\bigg(\sum_{\Delta_k > 0}\frac{\ln^2(T)}{V\Delta_k} + \frac{K\ln(T)}{V\Delta}\bigg)$ & $\frac{1}{1-2\alpha}\bigg(\sum_{\Delta_k > 0}\frac{\ln^2(T)}{v_i^w\Delta_k} + \frac{K\ln(T)}{v_{\min}^w\Delta}\bigg)$ \\
         \cline{2-4}
          & $\beta > \alpha$ & $\frac{1}{1-2\alpha}\bigg(\frac{C}{V} + \sum_{\Delta_k > 0}\frac{\ln^2(T)}{V\Delta_k} + \frac{K\ln(T)}{V\Delta}\bigg)$ & $\frac{1}{1-2\alpha}\bigg(\frac{C}{v_{\min}^w} + \sum_{\Delta_k > 0}\frac{\ln^2(T)}{v_i^w\Delta_k} + \frac{K\ln(T)}{v_{\min}^w\Delta}\bigg)$ \\
         \cline{1-4}
    \end{tabular}
    \caption{Overview of expected individual regret in multi-agent MAB with adversarial corruption. We omit constant terms that are independent of $T$. Notice that all the above algorithms in the centralized setting need a communication cost of $O(V\ln(T))$.}
    \label{tab:corruption}
\end{table*}
\begin{table}[t]
    \scriptsize
    \renewcommand\arraystretch{2}
    \begin{tabular}{|c|c|c|}
    \hline
    Methods & Individual regret & \makecell{Communication \\ cost} \\
    \hline
    \cite{zhu2023byzantine} & $\sum_{\Delta_k > 0}\frac{\ln(T)}{\Delta_k}$ & $VT$\\
    \cline{1-3}
    DeMABAR (Ours) & \makecell{$\frac{1}{1-2\alpha}\Big(\sum\limits_{\Delta_k > 0}\tfrac{\ln^2(T)}{v_i\Delta_k}
    + \tfrac{K\ln(T)}{v_{\min}\Delta}\Big)$} & $V\ln(T)$\\
    \cline{1-3}
    \end{tabular}
    \caption{Overview of expected individual regret and communication times in Byzantine DeCMA2B problems.}
    \label{tab:byzantine}
\end{table}
In this paper, we provide a positive answer to this question by proposing the DeMABAR (\textbf{De}centralized \textbf{M}ulti-\textbf{A}gent \textbf{B}andit \textbf{A}lgorithm with \textbf{R}obustness). Our method leverages the idea from the BARBAR algorithm~\cite{gupta2019better}, which is robust to adversarial corruptions in single-agent scenarios. Unlike BARBAR where the epoch length depends on the instance, our DeMABAR algorithm uses an instance-independent epoch length, ensuring that all agents have the same epoch length. In this way, DeMABAR allows agents to share information with their neighbors only at the beginning and the end of each epoch, thus improving communication efficiency. Theoretical analysis reveals that DeMABAR achieves a near-optimal regret bound in DeCMA2B under adversarial corruptions, with only a communication cost of $O(wV\ln(T))$. In addition, our DeMABAR includes a novel filtering mechanism to mitigate the influence of up to $\alpha V$ corrupted agents, where the hyperparameter $\alpha \in [0, \frac{1}{2})$ represents the fraction of malicious agents the system can tolerate. This filtering mechanism guarantees the robustness of DeMABAR in the presence of up to $\alpha V$ Byzantine agents.

We summarize the individual regret comparison for the adversarial corruption setting and the Byzantine setting in Table~\ref{tab:corruption} and Table~\ref{tab:byzantine}, respectively. The main contributions of this paper are summarized as follows:
\begin{itemize}
    \item We propose the novel DeMABAR algorithm for DeCMA2B and achieve near-optimal regret in both the adversarial corruption and Byzantine settings, with only a logarithmic communication cost.
    \item For DeCMA2B with adversarial corruptions, our DeMABAR algorithm achieves the following regret upper bounds for each agent $i$:
    \begin{align*}
        &\text{If $\beta \leq \alpha$, we have} \\
        &R_i(T){\le} O\Big(\frac{1}{1-2\alpha}\Big(\sum\limits_{\Delta_k > 0}\frac{\ln^2(T)}{v_i^w\Delta_k} + \frac{K\ln(T)}{v_{\min}^w\Delta}\Big)\Big), \\
        &\text{If $\beta > \alpha$, we have} \\
        &R_i(T){\le} O\Big(\frac{1}{1-2\alpha}\Big(\frac{C}{v_{\min}^w} {+} \sum\limits_{\Delta_k > 0}\frac{\ln^2(T)}{v_i^w\Delta_k} + \frac{K\ln(T)}{v_{\min}^w\Delta}\Big)\Big).
    \end{align*}
    The definitions of $v_i^w$ and $v_{\min}^w$ are introduced in the notation part of the next section.
    \item For DeCMA2B with Byzantine agents, our DeMABAR algorithm achieves the following regret bound for each agent $i$:
    \[
        R_i(T)\le O\left(\frac{1}{1-2\alpha}\bigg(\sum_{\Delta_k > 0}\frac{\ln^2(T)}{v_i\Delta_k} + \frac{K\ln(T)}{v_{\min}\Delta}\bigg)\right).
    \]
    The definitions of $v_i$ and $v_{\min}$ are introduced in the notation part of the next section.
    \item We also perform experiments to verify the robustness and effectiveness of our method.
\end{itemize}

\section{Related Work}
\label{sec:related work}
\paragraph{DeCMA2B.} 
Most prior works on DeCMA2B~\citep{martinez2019decentralized,lalitha2021bayesian,chawla2020gossiping} typically use gossip-based communication protocols to achieve consensus among agents. However, the algorithms in these works are \emph{not} robust to adversarial corruptions~\citep{jun2018adversarial,zuo2023adversarial}: even a small amount of adversarial corruption can cause such algorithms to suffer linear regret.

\paragraph{DeCMA2B with Adversarial Corruptions.} Adversarial corruptions in bandits were first considered by~\citet{lykouris2018stochastic}, and have attracted significant recent interest.~\citeauthor{lykouris2018stochastic} introduced an arm-elimination algorithm with regret scaling linearly in the total corruption $C$, and showed that a linear dependence on $C$ is unavoidable in general.~\citet{gupta2019better} proposed the BARBAR algorithm, which improves the dependence on $C$ by more judiciously sampling suboptimal arms. Building on this idea, several works have designed robust multi-agent bandit algorithms for adversarial corruption in \emph{centralized} settings~\citep{liu2021cooperative,ghaffari2024multi,hu2025near}, leveraging inter-agent collaboration to improve individual regret. However, to our knowledge, there is still no algorithm that is robust to adversarial corruptions in DeCMA2B.

\paragraph{DeCMA2B with Byzantine Agents.} Several recent works consider bandit learning in the presence of Byzantine agents.~\citet{madhushani2021one} studied an adaptive Byzantine communication model where any communicated reward can be arbitrarily altered.~\citet{vial2021robust,vial2022robust} and~\citet{mitra2022collaborative} considered settings where an unknown fraction of agents are Byzantine and can act arbitrarily. The approach of~\citeauthor{mitra2022collaborative} is specialized to linear contextual bandits and relies on a central coordinator, whereas~\citeauthor{vial2021robust} mitigate Byzantine influence by partitioning arms among agents (limiting the damage any single Byzantine can do). Most relevant to us,~\citet{zhu2023byzantine} were the first to propose a robust algorithm for DeCMA2B with Byzantine agents. They guarantee that the \emph{individual} regret of each normal (non-Byzantine) agent is strictly smaller than in the non-cooperative case; however, the improvement is only by a constant factor, rather than scaling inversely with the number of agents as is typical in benign cooperative settings. By contrast, our approach nearly retains the $\Theta(1/v_i)$ per-agent regret improvement even in the presence of Byzantine agents (see Table~\ref{tab:byzantine}).

\section{Preliminaries}
\label{sec:preliminaries}
In this section, we first describe the problem settings of multi-agent multi-armed bandits with the adversarial corruption and Byzantine settings. Then we introduce the notation used in this paper.

\subsection{Problem Setup}
\paragraph{Multi-agent Multi-armed Bandits}
Let $[V] = \{1, 2, \dots, V\}$ denote the set of $V$ agents and $[K] = \{1, 2, \dots, K\}$ denote the set of $K$ arms. The multi-agent network of $V$ agents is represented by the nodes of an undirected connected graph $G=([V],E)$, where $E$ is the set of edges. All agents face the same stochastic $K$-armed bandit problem over a horizon of $T$ rounds. In each round $t$, every agent $i$ selects an arm $k_{i,t}$ and receives a reward $r_{i,t}$ that is drawn i.i.d. from a fixed but unknown distribution with mean $\mu_{k_{i,t}} \in [0,1]$. After obtaining the reward, each agent may broadcast messages to its neighbors and receive messages from its neighbors. The received information can be used in the next round if desired.

Let $k^* \in \argmax_k \mu_k$ be an optimal arm, and we define $\Delta_k = \mu_{k^*} - \mu_k$ as the suboptimality gap of arm $k$, and let $\Delta = \min_{\Delta_k>0} \Delta_k$ be the smallest positive suboptimality gap. Let $n^k_{i,t}$ be the number of times that agent $i$ has pulled arm $k$ up to round $t$. The individual pseudo-regret of agent $i$ over $T$ rounds is defined as
\[ R_i(T) = T \mu_{k^*} - \BE\bigg[\sum_{t=1}^T r_{i,t}\bigg] 
= \sum_{k=1}^K \Delta_k \BE[n^k_{i,T}]\,.\]

For simplicity, we quantify communication cost as the total number of messages broadcast by all agents. The total communication cost over $T$ rounds is defined as
\[ \textstyle \mathrm{Cost}(T) = \sum_{i=1}^V \sum_{t=1}^T \mathbb{I}\{\text{agent $i$ broadcasts at time $t$}\}. \]

\paragraph{Adversarially corrupted setting} In this setting, at each round~$t\in[T]$, the protocol between the agents and the adversary is as follows:
\begin{enumerate}
    \item The environment generates a reward vector $(r_{i,t}(1), \dots, r_{i,t}(K))$ for each agent $i$, according to the reward distributions.
    \item The adversary observes all reward vectors and generates a \emph{corrupted} reward vector $(\tilde r_{i,t}(1), \dots, \tilde r_{i,t}(K))$ for each agent $i$, based on the history of the previous $t-1$ rounds.
    \item Each agent $i$ chooses an arm $k_{i,t}$ and observes only the corrupted reward $\tilde r_{i,t}(k_{i,t})$ for that arm.
\end{enumerate}
The corruption level of the adversary is defined as 
\[ C \;=\; \sum_{i=1}^V \sum_{t=1}^T \max_{k \in [K]} \big|\,\tilde r_{i,t}(k) - r_{i,t}(k)\,\big|. \]
We assume that for each agent $i$, the adversary can corrupt at most a fraction $\beta \in [0,1]$ of its neighbors. Note that both $C$ and $\beta$ are \emph{unknown} to the agents.
\paragraph{Byzantine setting} In the Byzantine agent model, a subset of the agents (called Byzantine agents) may act adversarially. A Byzantine agent can select arbitrary arms in each round and send arbitrary messages to its neighbors, potentially sending different messages to different neighbors. Normal agents do not know which of their neighbors are Byzantine, but for each normal agent $i$, we assume that at most a fraction $\alpha \in [0, 0.5)$ of its neighbors are Byzantine. As in previous work on the Byzantine model~\cite{vial2021robust,vial2022robust,mitra2022collaborative,zhu2023byzantine}, we assume that $\alpha$ is known to the algorithm. In the Byzantine setting, we only focus on the regret of the normal agents, since Byzantine agents can behave arbitrarily.

\paragraph{Relationship between adversarially corrupted and Byzantine settings} In the adversarially corrupted setting, an adversary can manipulate the rewards generated by the environment with a budget $C$. Conversely, Byzantine agents can send arbitrary information to other agents in every round, acting as if they had an infinite corruption budget. Additionally, in the Byzantine setting, we only consider the individual regret of normal agents, while in the adversarially corrupted setting, we consider the individual regret of all agents.

\subsection{Notation}
Given a graph $G = ([V], E)$, we let $d(u,v)$ denote the number of edges of a shortest path connecting nodes $u$ and $v$ in $G$. Note that we have $d(v,v)=0$ for any node $v$. For an integer $w \ge 0$, we define $\cN_w(i) = \{\, j \in V : d(i,j) \le w \,\}$ as the set of nodes located within distance $w$ from node $i$, which is also referred to as the $w$-neighborhood of node $i$. Note that we have $\{i\} = \cN_0(i) \subseteq \cN_1(i) \subseteq \cN_2(i) \subseteq \cdots$. Let $D = \max_{u,v \in [V]}d(u,v)$ denote the diameter of the graph $G$. We define $v_i^w = \min_{j \in \cN_w(i)} |\cN_w(j)|$ as the minimum number of nodes within distance $w$ of any node in the $w$-neighborhood of node $i$. We define $v_{\min}^w = \min_{j \in [V]} |\cN_w(j)|$ as the smallest $w$-neighborhood size among all nodes, so we have $v_{\min}^w = \min_{i \in [V]}v_i^w$. For simplicity, we define $v_i = v_i^1$ and $v_{\min} = v_{\min}^1$.

\section{Algorithm}
\label{sec:algorithm}
In this section, we present our robust algorithm DeMABAR, summarized in Algorithm~\ref{algs:DeMABAR}. For clarity, we first consider the adversarial corruption setting, followed by the Byzantine agent model.

\subsection{DeCMA2B with Adversarial Corruptions}
Our DeMABAR algorithm operates in synchronized epochs, with agents allowed to broadcast information at the end of each epoch. We denote by $w\in[D]$ the \emph{collaboration distance}, meaning each agent $i$ will exchange messages with the agents in its $w$-neighborhood $\cN_w(i)$, which incurs a delay of $w - 1$ rounds for information to propagate $w$ hops. The hyperparameter $\alpha \in [0,0.5)$ serves as an estimate of $\beta$, the maximum fraction of corrupted agents among any node’s neighbors.

At the start of epoch $m$, each agent $i$ computes an empirical suboptimality-gap estimate $\Delta^{m-1}_{i,k}$ for all arms $k\in[K]$, based on data from the previous epoch. Each arm $k$ is expected to be pulled about $(\Delta^{m-1}_{i,k})^{-2}$ times, but capped by $2^{2m}$ pulls to avoid over-exploring any arm. For each agent $i$, all agents $j$ in its $w$-neighborhood are responsible for roughly a $\frac{1}{(1-2\alpha)|\cN_w(i)|}$ fraction of the pulls for each arm $k$, i.e., 
$n^m_{i,k} = \frac{\lambda \, (\Delta^{\,m-1}_{i,k})^{-2}}{(1-2\alpha)\,|\cN_w(i)|}$.
However, to satisfy the collaboration requirements of all agents in $i$’s $w$-neighborhood, it may need to pull slightly more than $n^m_{i,k}$ times; thus, we define $\tilde n^m_{i,k}$ (line 7) as the expected number of pulls for agent $i$ on arm $k$ in epoch $m$. On the other hand, if $\sum_{k=1}^K \tilde n_{i,k}^m < N_m$, we select the arm $k_i^m$ that exhibited the best performance in the previous epoch and adjust its number of pulls to be $\tilde n_{i,k_i^m}^m$. After $N_m$ rounds, each agent $i$ broadcasts the received information $(i, \{S_{i,k}^m\}_{k=1}^K, \{\tilde n_{i,k}^m\}_{k=1}^K)$ to its $w$-neighborhood, and this step requires $w$ rounds. This communication process requires $w$ rounds, during which each agent $i$ selects the arm $k_i^m$ but does not record the received reward. At the end of epoch $m$, each agent $i$ uses Algorithm~\ref{alg:filter} to filter out corrupted data before the next epoch’s estimates are computed.
\begin{algorithm}[t]
    \caption{\textbf{DeMABAR}}
    \label{algs:DeMABAR}
    \begin{algorithmic}[1]
        \STATE \textbf{Input:} collaboration distance $w$, fraction $\alpha \in [0,0.5)$.
        \STATE \textbf{Initialize:} $T_0 \leftarrow 0$, $\Delta^0_{i,k} \leftarrow 1$, and $\lambda \leftarrow 2^9 \ln(2VT)$.
        \FOR{all agent $i\in[V]$ in parallel}
        \FOR{epoch $m = 1,2,\ldots$}
            \STATE $N_m \leftarrow \left\lceil \frac{\lambda K \, 2^{\,m-1}}{(1-2\alpha)\,v_{\min}^w} \right\rceil$,\;\; $T_m \leftarrow T_{m-1} + N_m$.  
            \STATE $ \tilde n^m_{i,k} \leftarrow \min\Big\{\frac{16\,\lambda\,(\Delta^{\,m-1}_{i,k})^{-2}}{(1-2\alpha)\,v_i^w}, \frac{\lambda\,2^{\,2(m-1)}}{(1-2\alpha)\,v_i^w}\Big\}$.
            \STATE Select arm $k^m_i$ such that $\Delta^{\,m-1}_{i,k^m_i} = 2^{-(m-1)}$.
            \STATE Set $\tilde n^m_{i,\,k^m_i} \leftarrow N_m- \sum_{k \ne k^m_i} \tilde n^m_{i,k}$.
            \FOR{$t = T_{m-1} + 1$ {\bf to} $T_m$}
                \STATE Pull arm $k_{i,t} \sim p^m_i$, where $p^m_i(k) = \tilde n^m_{i,k} / N_m$.
                \STATE Observe corrupted reward $\tilde r_{i,t}(k_{i,t})$.
                \STATE Update $S^m_{i,k_{i,t}} \leftarrow S^m_{i,k_{i,t}} + \tilde r_{i,t}(k_{i,t})$.
            \ENDFOR
            \STATE\textit{Communication step:} 
            \FOR{$t = T_m + 1$ {\bf to} $T_m + w$}  
                \STATE Pull arm $k_i^m$ and observe corrupted reward.
                \STATE Send message $( i, \{S_{i,k}^m\}_{k=1}^K, \{\tilde n_{i,k}^m\}_{k=1}^K)$ and all messages received at $t-1$ round to neighbors.
                \STATE Receive messages from the neighboring agents.
            \ENDFOR
            \STATE \textit{Filter step:} Run Algorithm~\ref{alg:filter} to obtain $r_{i,k}^m$.
            \STATE Set $T_m \leftarrow T_m + w$.
            \STATE Set $r^m_{i,*} \leftarrow \max_{k}\{r^m_{i,k} - \frac{1}{8}\Delta^{m-1}_{i,k}\}$.
            \STATE Set $\Delta^m_{i,k} \leftarrow \max\{2^{-m}, r^m_{i,*} - r^m_{i,k}\}$ for each arm $k$.
            \ENDFOR
            \ENDFOR
    \end{algorithmic}
\end{algorithm}
\paragraph{Filtering mechanism} 
First, for each arm $k$, each agent $i$ removes the messages from agents whose number of pulls $\tilde n^m_{j,k}$ is lower than $n^m_{i,k}$. If too many neighbors are removed, leaving fewer than $(1-2\alpha)|\cN_w(i)|$ neighbors, we define this event as $\cL^m_{i,k}$, and agent $i$ resets $n^m_{i,k}$ to the minimum observation count and restores all neighbors. Then, each agent $i$ sets $\cB^m_{i,k} = \cA^m_{i,k}$ as the set of neighbors whose data will be used for its final estimate on arm $k$. Each agent $i$ sorts $S^m_{j,k}/\tilde n^m_{j,k}$ in descending order for $j \in \cB_{i,k}^m$ and removes the indices corresponding to the $f$ largest and $f$ smallest values from $\cB_{i,k}^m$. Finally, it uses the filtered data to compute $r^m_{i,k}$ as the trimmed average.
\begin{algorithm}[t]
    \caption{\textbf{Filter} (agent $i$)}
    \label{alg:filter}
    \begin{algorithmic}[1]    
        \FOR{arm $k = 1,2,\ldots,K$}
            \STATE Initialize an available set $\cA_{i,k}^m = \cN_w(i)$.
            \STATE Remove agent $j$ from $\cA^m_{i,k}$ if $\tilde n^m_{j,k} > n^m_{i,k}$.
            \IF{$|\cA_{i,k}^m| < (1-2\alpha) |\cN_w(i)|$} 
                \STATE Set $n^m_{i,k} \leftarrow \min\limits_{j \in \cN_w(i)} \tilde n^m_{j,k}$, and $\cA^m_{i,k} \leftarrow \cN_w(i)$.
            \ENDIF
            \STATE Let $\cB^m_{i,k} \leftarrow \cA^m_{i,k}$.
            \STATE  Let $f = \frac{1}{2}\big\lfloor(|\cB_{i,k}^m| - (1-2\alpha)|\cN_w(i)|)\big\rfloor$.
            \STATE Sort $S^m_{j,k}/\tilde n^m_{j,k}$ in descending order for $j \in \cB_{i,k}^m$ and remove the indices corresponding to the $f$ largest and $f$ smallest values from $\cB_{i,k}^m$.
            \STATE Set $
                r^m_{i,k} \leftarrow \min\!\big\{\frac{1}{|\cB_{i,k}^m|} \sum_{j \in \cB^m_{i,k}} \frac{S^m_{j,k}}{\tilde n^m_{j,k}}, 1\big\}. 
            $
        \ENDFOR
        \STATE \textbf{Output:} $r^m_{i,k}$ for $k\in[K]$.
    \end{algorithmic}
\end{algorithm}
\paragraph{Robustness when $\boldsymbol{\beta \leq \alpha}$} 
When the fraction of adversarially corrupted neighbors satisfies $\beta \leq \alpha$, at most $\big\lfloor \alpha |\cN_w(i)|\big\rfloor$ of agent $i$’s neighbors can be corrupted in any epoch. Consequently, the set $\cB^m_{i,k}$ (after filtering) will contain exactly $|\cN_w(i)|$ agents, resulting in $f = \lfloor \alpha |\cN_w(i)| \rfloor$. Even if up to $f$ corrupted agents remain in $\cB^m_{i,k}$, their average rewards $S^m_{j,k}/\tilde n^m_{j,k}$ will be bounded above and below by at least an equal number of uncorrupted agents. Consequently, removing the $f$ largest and $f$ smallest values ensures that any artificially inflated or deflated contributions from corrupted agents are eliminated, which guarantees that the estimate $r^m_{i,k}$ is close to the true mean reward.
\paragraph{Robustness when $\boldsymbol{\beta > \alpha}$} 
If the adversary can corrupt more than an $\alpha$ fraction of its neighbors, DeMABAR still maintains robustness by never permanently eliminating any arm based on possibly corrupted data. Instead, it continues to occasionally explore every arm, allocating a limited number of pulls to each arm in each epoch. Our DeMABAR algorithm guarantees that a corruption amount of $C_m$ in epoch $m$ will only result in $O(C_m 2^{-(s-m)})$ additional pulls for all suboptimal arms in subsequent epochs $s>m$. Thus, the regret of our DeMABAR algorithm will only suffer an additive term that depends on the total corruption budget $C$.

We show that the DeMABAR algorithm achieves the following regret bound, and the proof is deferred to the Appendix.
\begin{theorem}\label{the:DeMABAR}
In DeCMA2B with adversarial corruptions, our DeMABAR algorithm only requires a communication cost of $O(wV \ln(VT))$ to achieve the following individual regret for each agent $i$:\\
If $\beta \le \alpha$, we have
{\small \begin{align*}
    R_i(T) = O\left(\frac{\ln(VT)}{1-2\alpha}\bigg(\sum_{\Delta_k>0}\frac{\ln(VT)}{v_i^w\Delta_k} + \frac{K\ln(VT)\ln(\frac{1}{\Delta})}{v_{\min}^w\Delta}\bigg)\right),
\end{align*}}
If $\beta > \alpha$, we have
{\small \begin{align*}
    R_i(T) &= O\left(\frac{\ln(VT)}{1-2\alpha}\bigg(\sum_{\Delta_k>0}\frac{\ln(VT)}{v_i^w\Delta_k} + \frac{K\ln(VT)\ln(\frac{1}{\Delta})}{v_{\min}^w\Delta}\bigg)\right)\\
    &\quad \quad + O\bigg(\frac{C}{(1-2\alpha)v_{\min}^w}\bigg).
\end{align*}}
\end{theorem}
\begin{remark}
    In the case $\beta \le \alpha$, the regret bound in Theorem~\ref{the:DeMABAR} is independent of $C$. Even a strong adversary~\cite{zuo2024near} cannot force the DeMABAR algorithm to suffer linear regret.
\end{remark}
\begin{corollary}
For centralized CMA2B with adversarial corruptions and $\beta = 1$, our DeMABAR algorithm with $\alpha = \frac{1}{3}$ has the following individual regret for each agent $i$:
\begin{align*}
    R_i(T) = O\left(\frac{C}{V} +\sum_{\Delta_k>0}\frac{\ln^2(VT)}{V\Delta_k} + \frac{K\ln(VT)\ln(\frac{1}{\Delta})}{V\Delta}\right).
\end{align*} 
\end{corollary}
\begin{remark}
As highlighted in Table~\ref{tab:corruption}, even in this setting, our algorithm’s individual regret bound is strictly smaller (by logarithmic factors or more) than prior results for robust multi-agent bandits~\cite{liu2021cooperative,ghaffari2024multi}. Compared with~\cite{hu2025near}, the main part of our regret bound is consistent with theirs.
\end{remark}
\subsection{DeCMA2B with Byzantine agents}
Recalling the setup of the Byzantine setting, for each normal agent, at most a fraction $\alpha \in [0, 0.5)$ of its neighbors are Byzantine agents. In the Byzantine setting, we think communication at distances greater than $1$ is inherently unsafe because Byzantine agents may maliciously modify the received messages and send wrong information to their neighbors.
\begin{figure*}[t]
    \centering
    \renewcommand{\arraystretch}{1.5}
    \includegraphics[width=\linewidth]{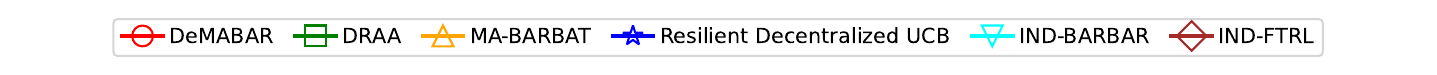}
    \begin{tabular}{cccc}
        \includegraphics[width = 0.22\textwidth]{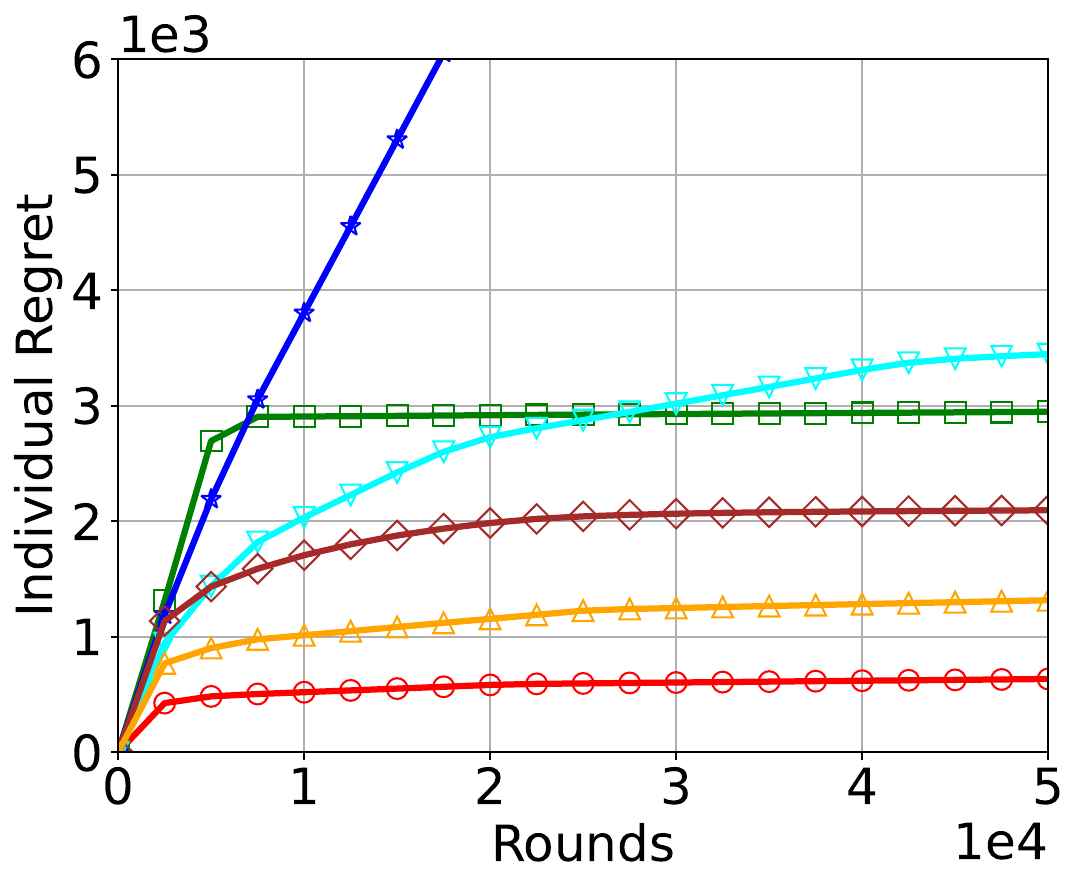} &  
        \includegraphics[width = 0.22\textwidth]{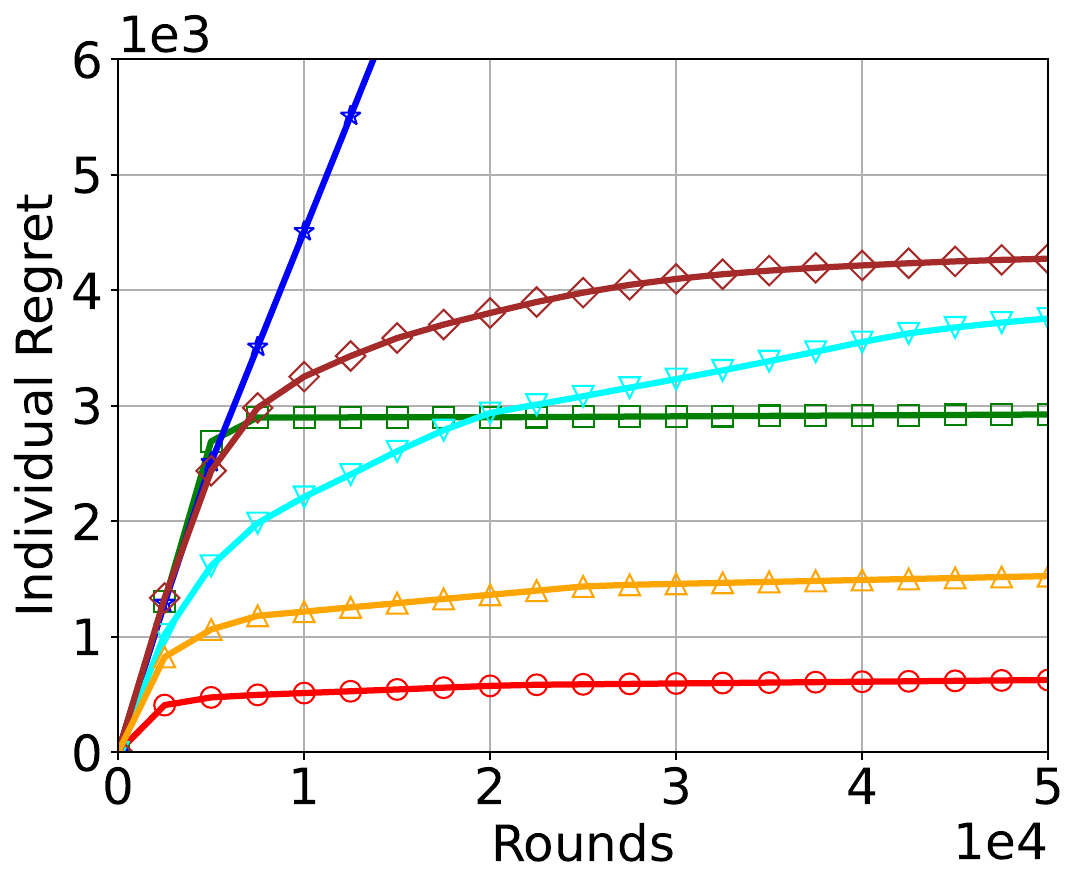} &
        \includegraphics[width = 0.22\textwidth]{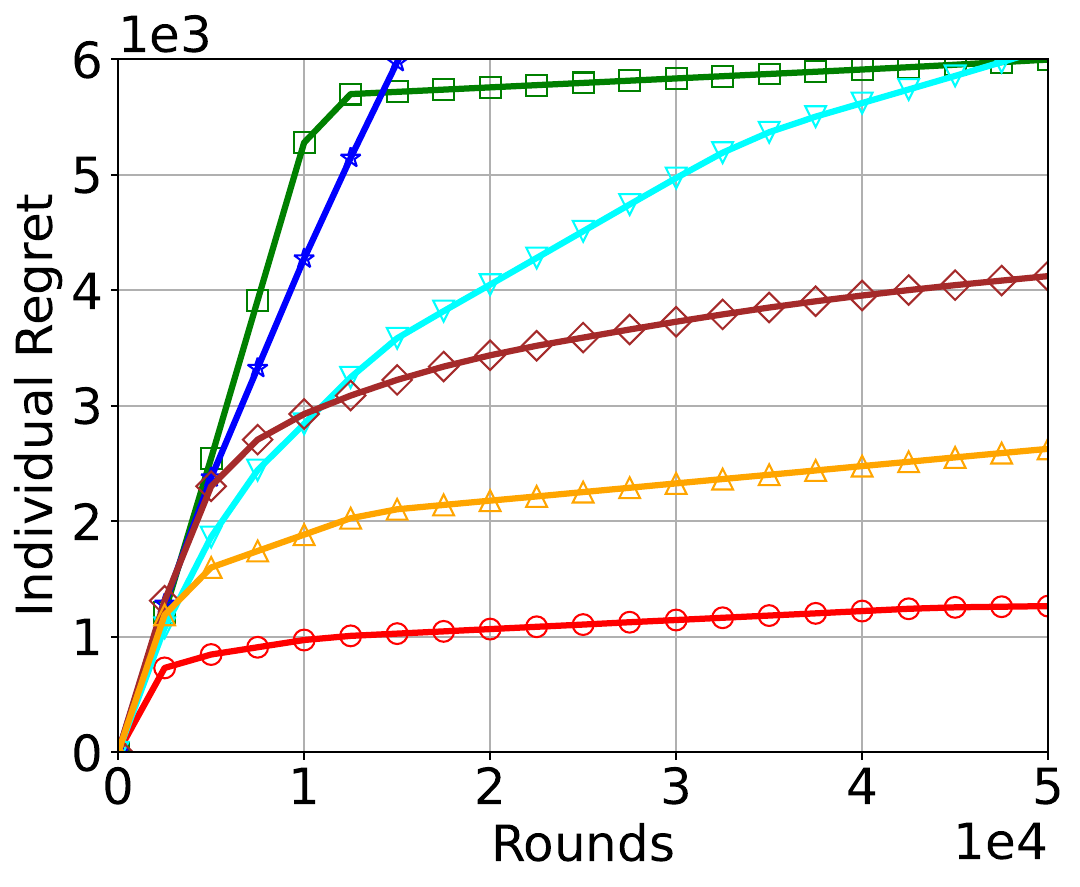} & 
        \includegraphics[width = 0.22\textwidth]{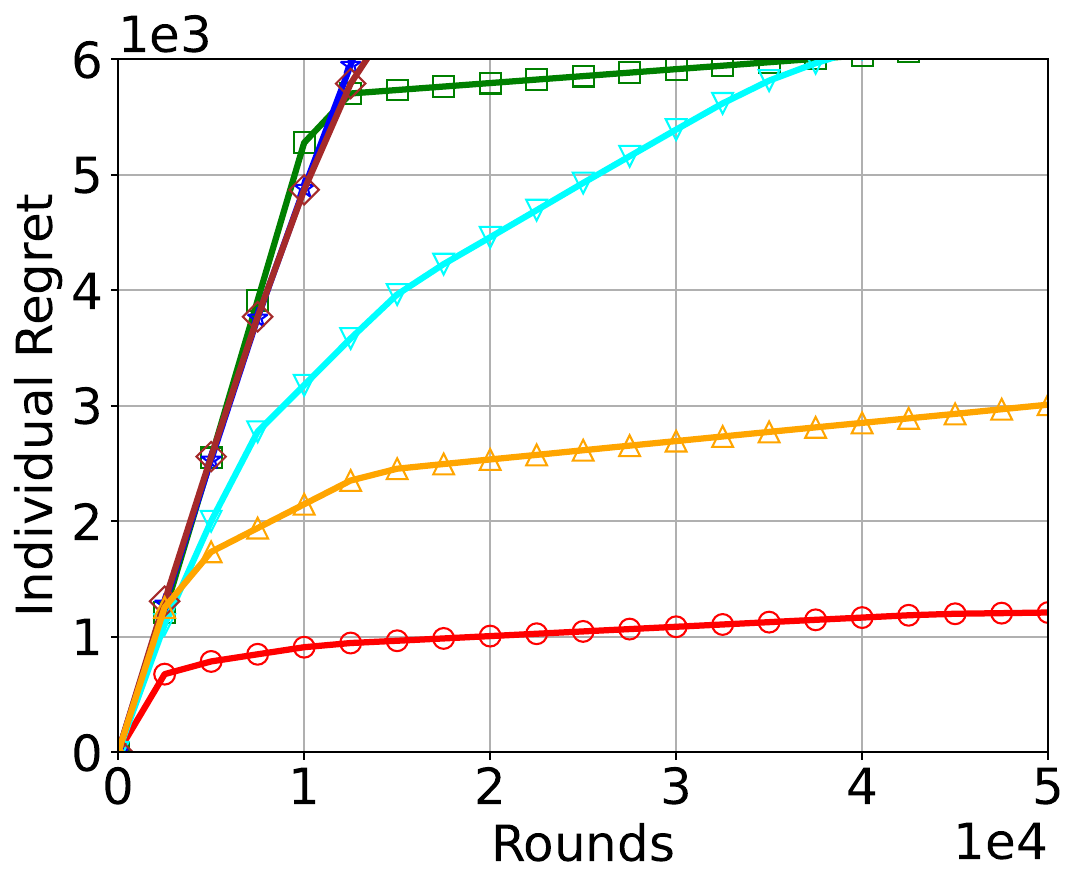}\\
        (a)  K = 10, C = 1500 &
        (b)  K = 10, C = 2000 &
        (c)  K = 20, C = 3000 &
        (d)  K = 20, C = 4000 \\
        \multicolumn{4}{c}{\textbf{(i) The adversary attacks all agents.}} \\
        \includegraphics[width = 0.22\textwidth]{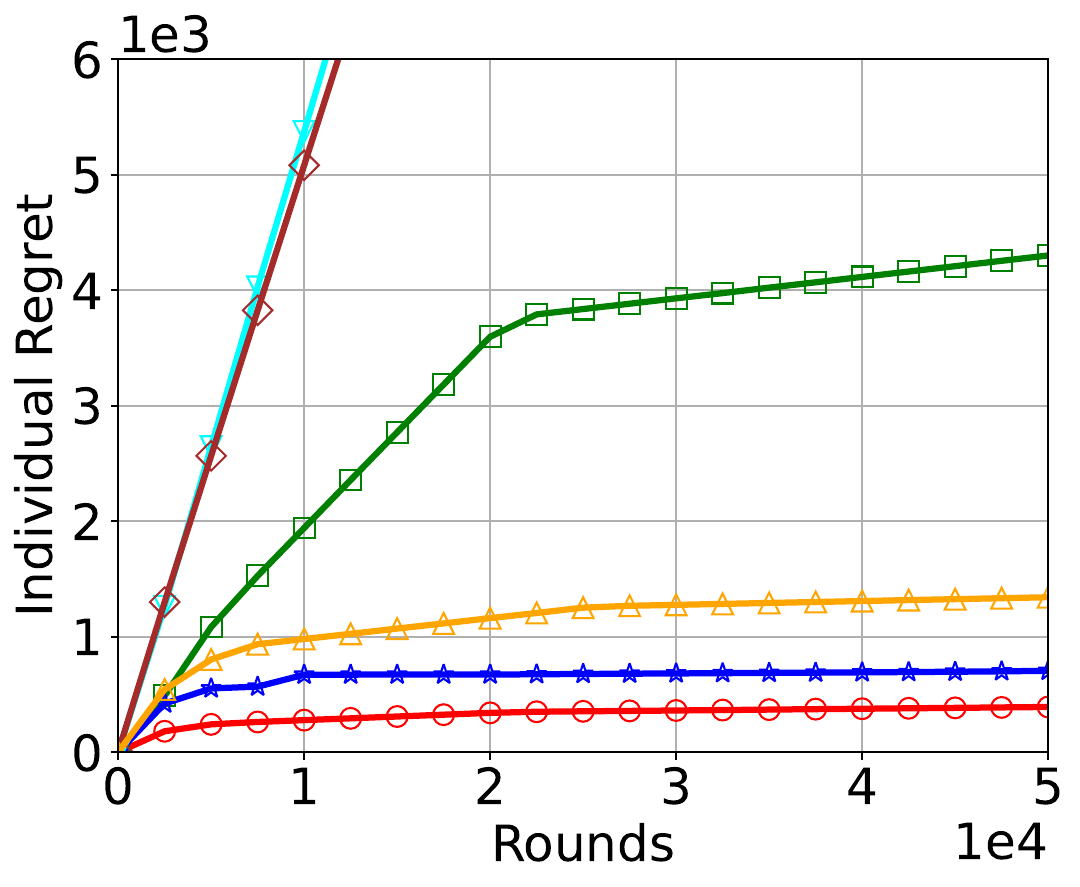} &  
        \includegraphics[width = 0.22\textwidth]{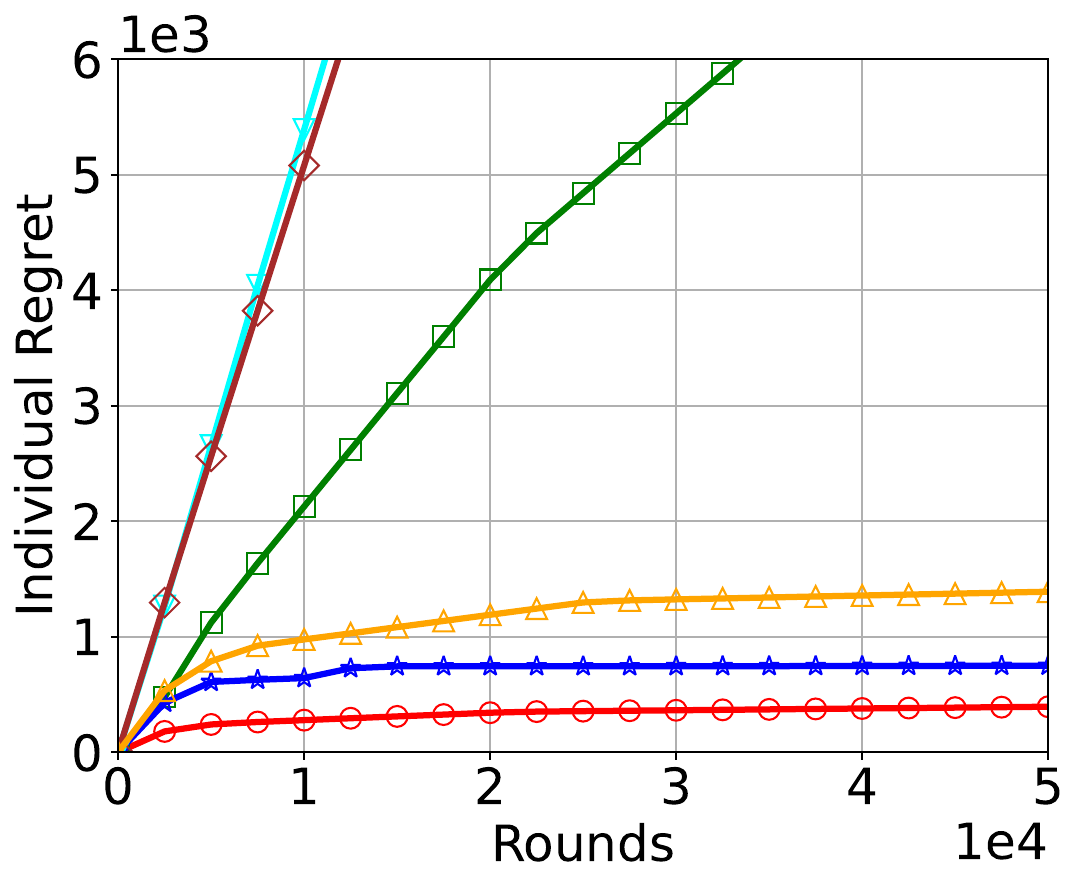} &
        \includegraphics[width = 0.22\textwidth]{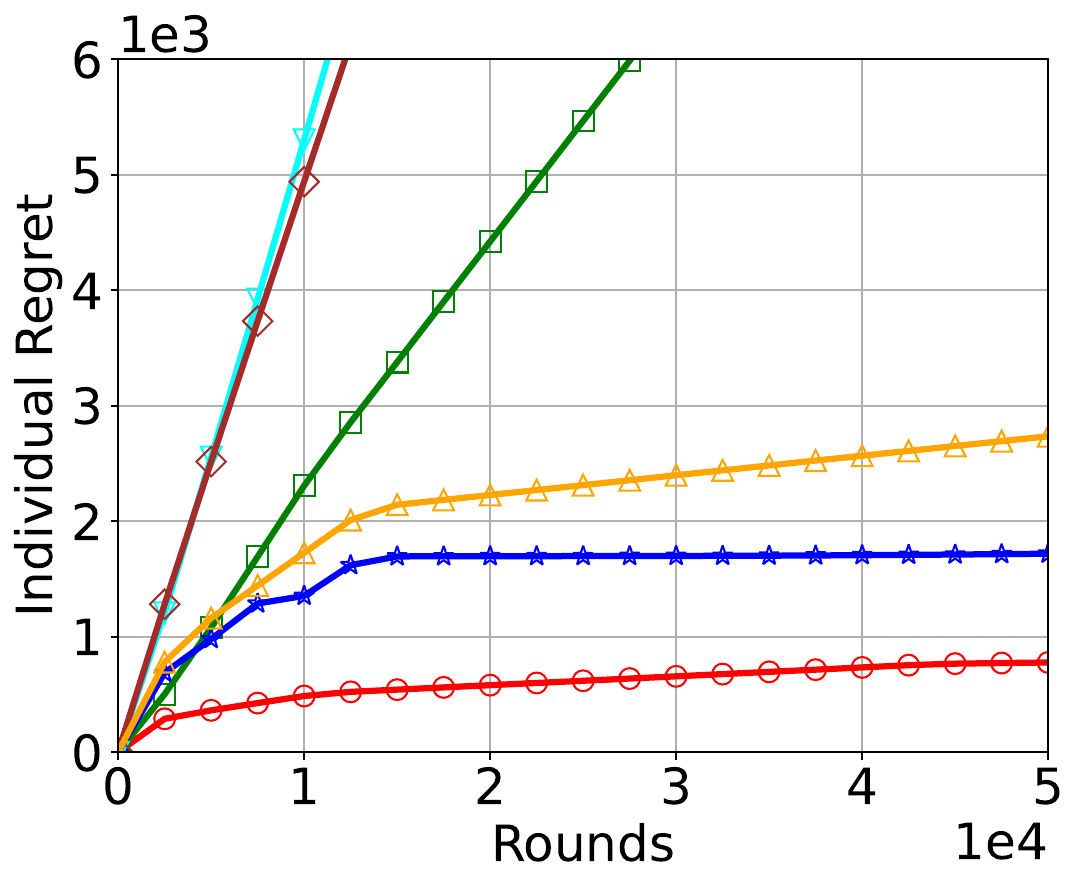} & 
        \includegraphics[width = 0.22\textwidth]{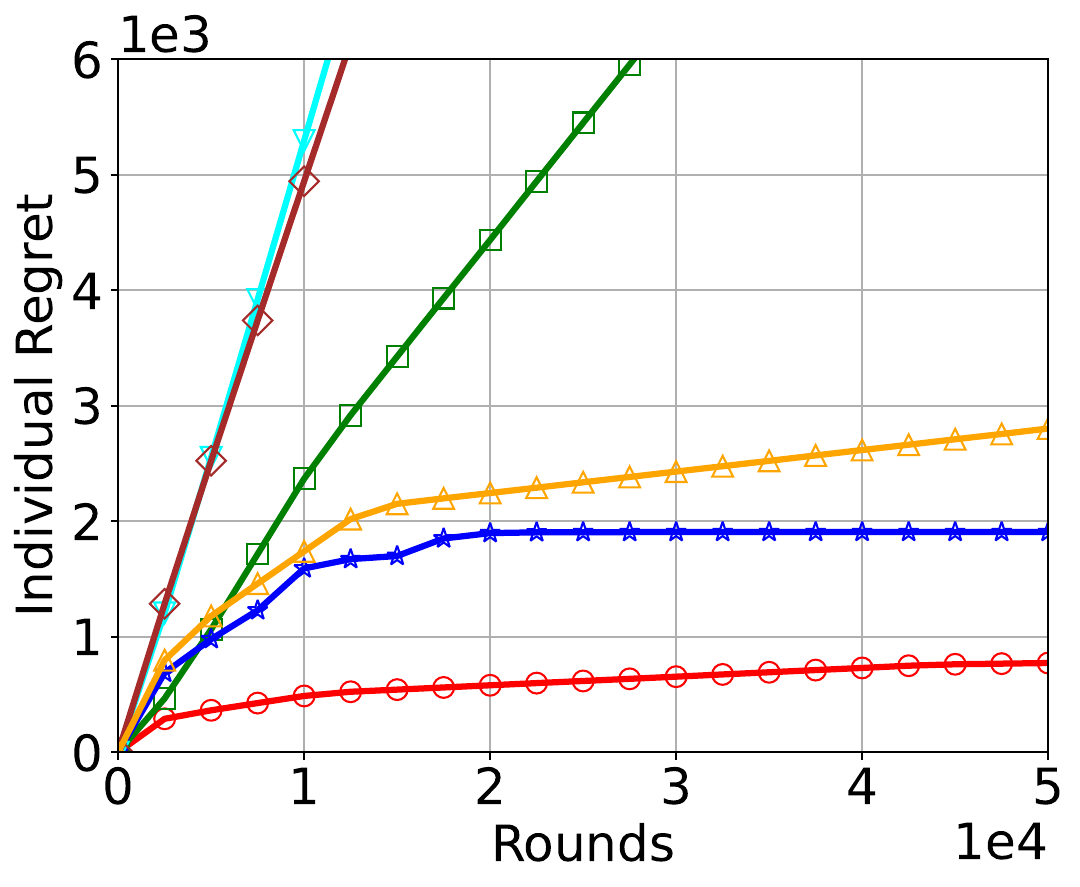}\\
        (a)  K = 10, C = 6000 &
        (b)  K = 10, C = 8000 &
        (c)  K = 20, C = 12000 &
        (d)  K = 20, C = 16000 \\
        \multicolumn{4}{c}{\textbf{(ii) The adversary attacks three agents.}} 
    \end{tabular}
    \caption{DeMABAR vs. DRAA, Resilient Decentralized UCB, MA-BARBAT, IND-BARBAR, and IND-FTRL in centralized CMA2B under adversarial corruption.}
    \label{fig:cor-dis}
\end{figure*}
\begin{figure}[t]
    \centering
    \includegraphics[width=0.4\linewidth]{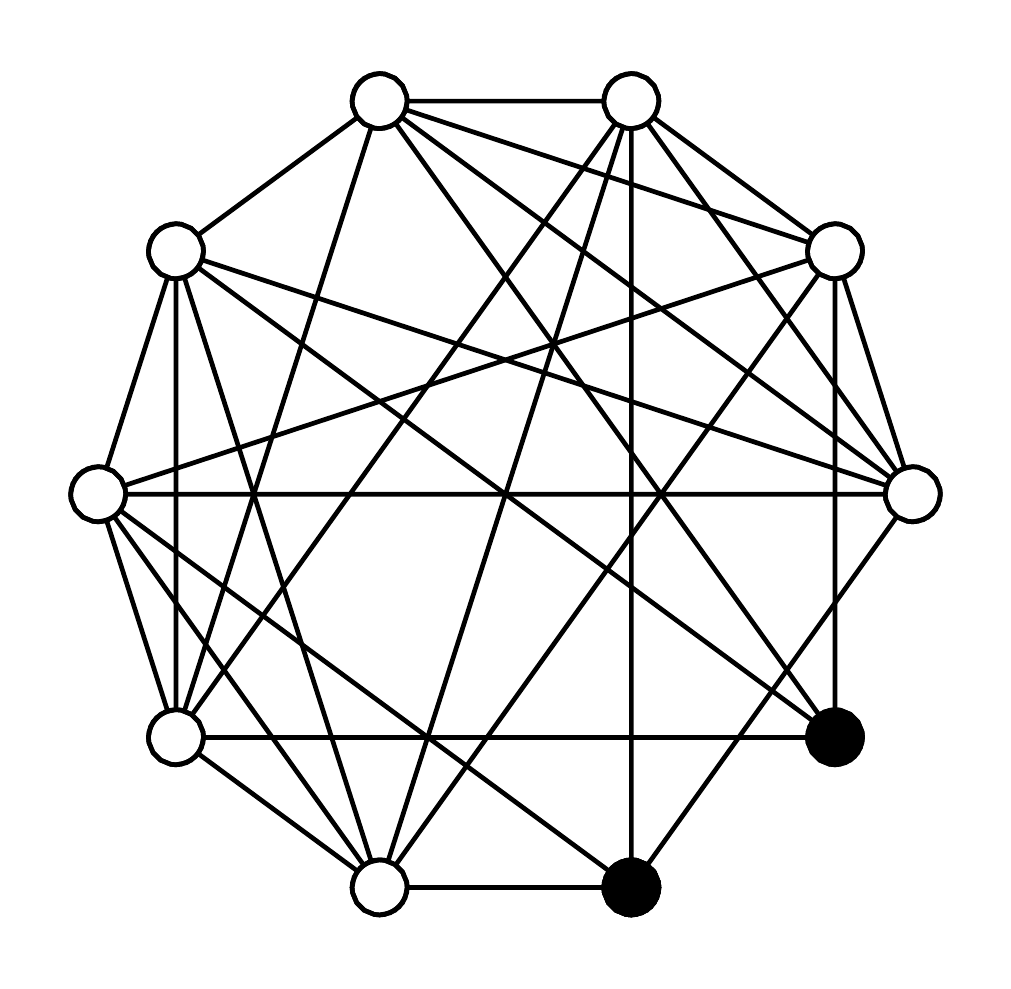}
    \caption{The network structure used in the experiment.}
    \label{fig:gra}
\end{figure} 
Thus we set $w = 1$, meaning that each agent only receives messages from its immediate neighbors. This method that leverages only one-hop neighbor information has also been employed in prior studies~\citep{zhu2023byzantine,wang2023online,liu2025offline} to preserve robustness.

Fortunately, the DeMABAR algorithm described earlier already includes two filtering mechanisms (in Algorithm~\ref{alg:filter}, lines 4-5 and 12-13) specifically designed to handle potentially Byzantine neighbors. These filters ensure that even if some neighbors are Byzantine agents, their influence on an agent’s estimates is negligible.

We can thus bound the regret of our algorithm in the Byzantine setting. The proof is deferred to the Appendix.
\begin{theorem}\label{the:regret case3}
    In DeCMA2B with Byzantine agents, with $O(V \ln(VT))$ communication cost, Algorithm~\ref{algs:DeMABAR} achieves the following regret for each normal agent $i$:
    \[ 
        R_i(T) = O\left(\frac{\ln(VT)}{\,1 - 2\alpha\,}\bigg(\;\sum_{\Delta_k>0} \frac{\ln(VT)}{v_i\,\Delta_k} + \frac{K\,\ln(\frac{1}{\Delta})}{\,v_{\min}\,\Delta\,}\bigg)\right).
    \]
\end{theorem}
\begin{remark}
    For the case of $\alpha \le 1/3$, \citeauthor{zhu2023byzantine}(\citeyear{zhu2023byzantine}) achieve regret of $\sum_{\Delta_k>0} \frac{\ln T}{\Delta_k}$. In contrast, our regret bound in Theorem~\ref{the:regret case3} has additional $v_i$ and $v_{\min}$ terms in the denominators, explicitly quantifying the benefit of collaboration.
\end{remark}
\begin{remark}
    \citeauthor{zhu2023byzantine}(\citeyear{zhu2023byzantine}) show that for sufficiently large $T$, the regret of DeCMA2B with Byzantine agents satisfies
    \[ 
        R_i(T) \ge \Omega\Bigg(\sum_{\Delta_k>0} \frac{\ln(T)}{(\,1-2\alpha\,)|\cN_1(i)|\,\Delta_k}\Bigg)\,. 
    \]
    There remains a gap between this lower bound and our upper bound. Closing this gap or determining if it is unavoidable is an interesting open question for future work.
\end{remark}

\section{Experiments}
\label{sec:exp}
\begin{figure*}[t]
    \centering
    \renewcommand{\arraystretch}{1.5}
    \includegraphics[width = \textwidth]{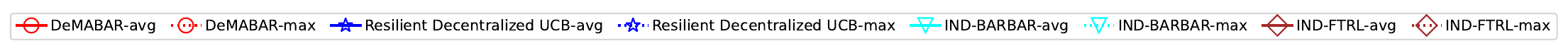}
    \begin{tabular}{cccc}
        \includegraphics[width = 0.22\textwidth]{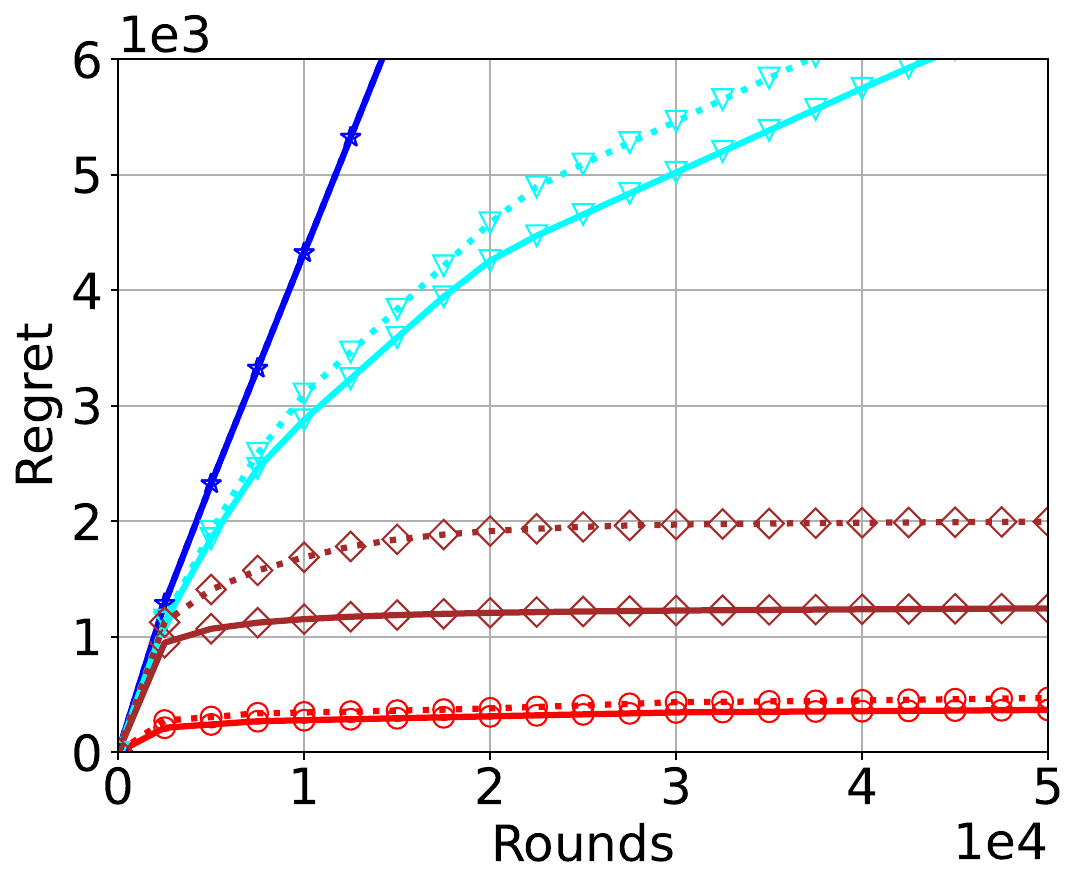} &  
        \includegraphics[width = 0.22\textwidth]{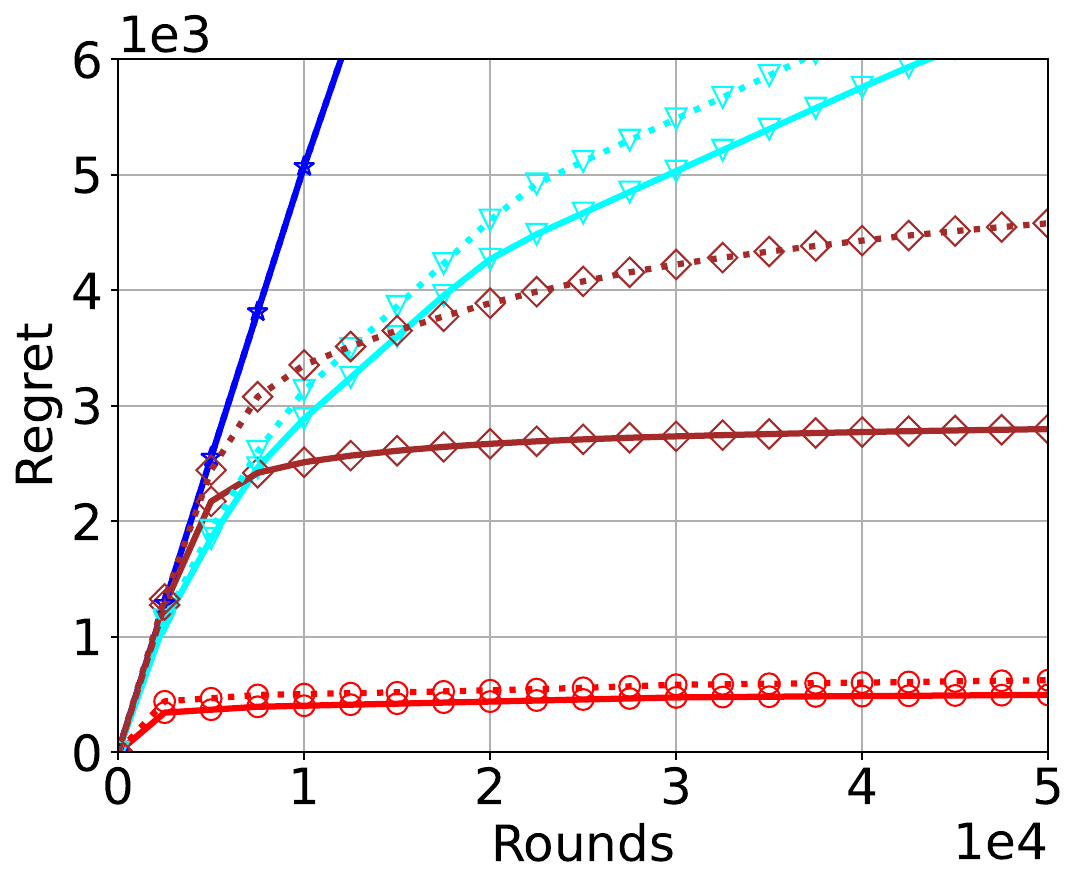} &
        \includegraphics[width = 0.22\textwidth]{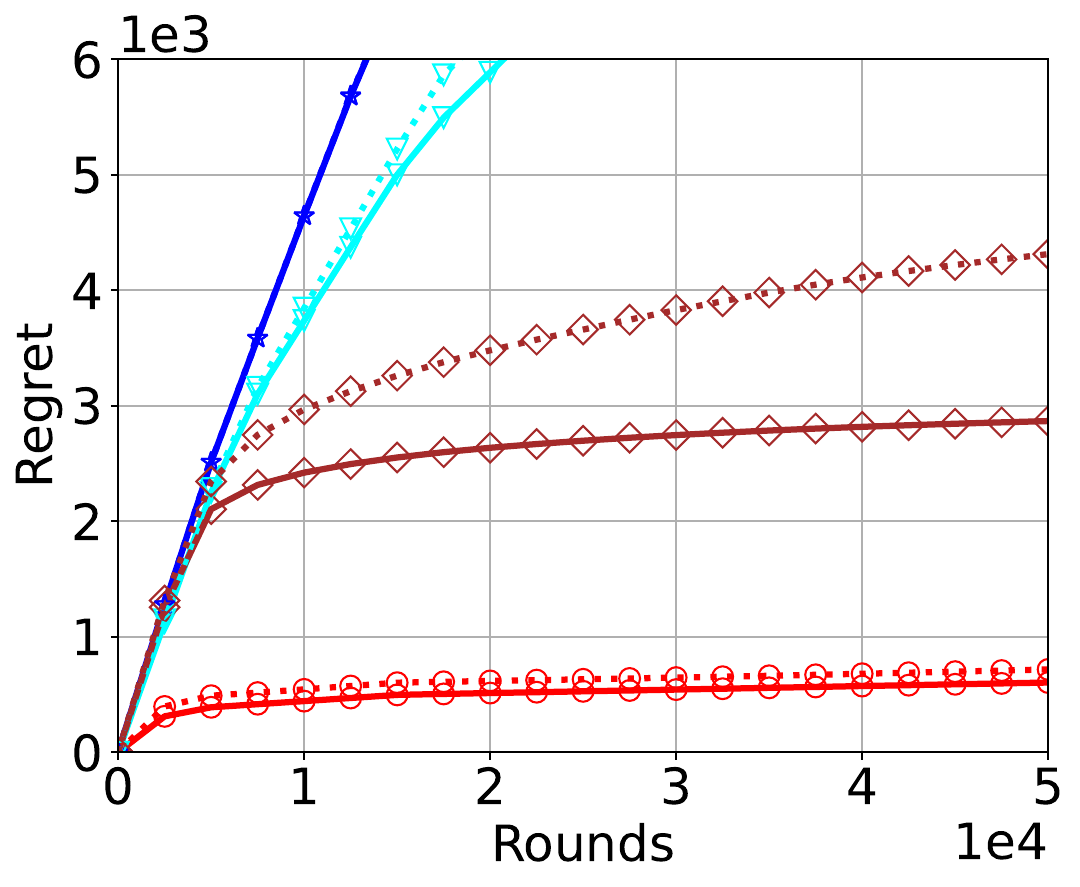} & 
        \includegraphics[width = 0.22\textwidth]{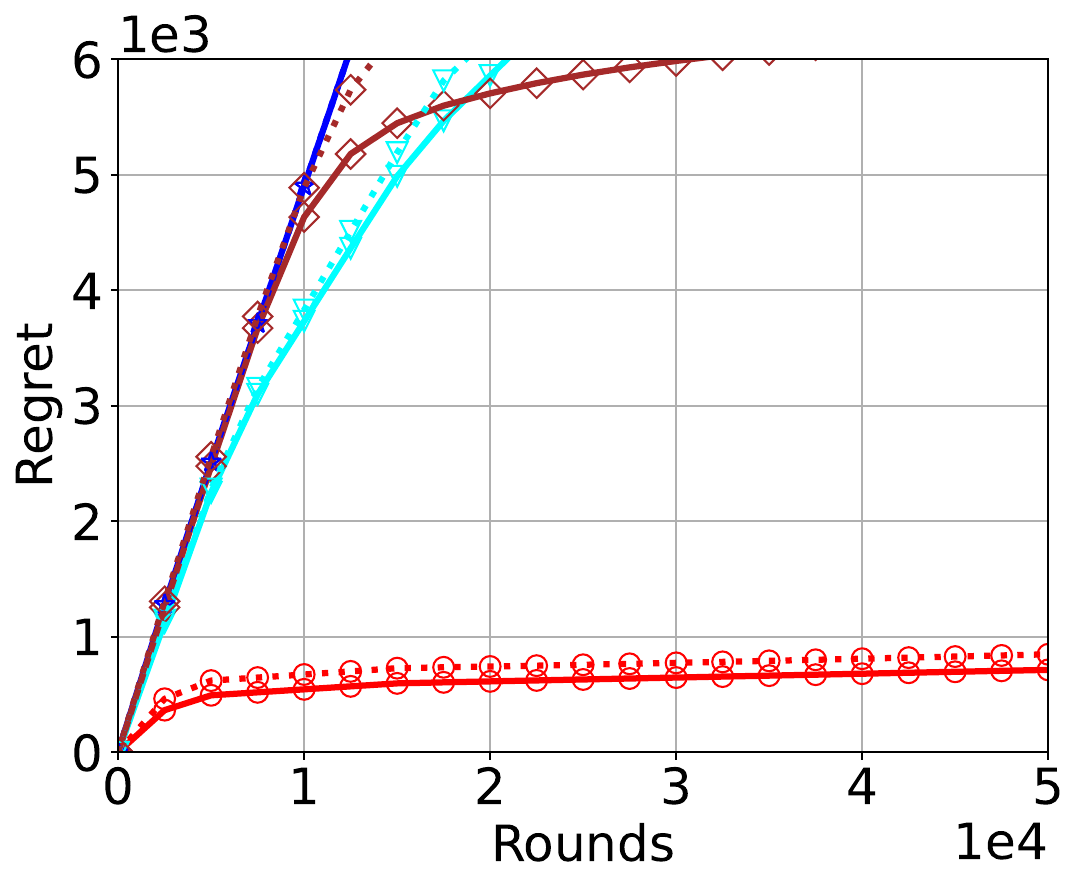}\\
        (a)  K = 10, C = 1500 &
        (b)  K = 10, C = 2000 &
        (c)  K = 20, C = 3000 &
        (d)  K = 20, C = 4000 \\
        \multicolumn{4}{c}{\textbf{(i) The adversary attacks all agents.}} \\
        \includegraphics[width = 0.22\textwidth]{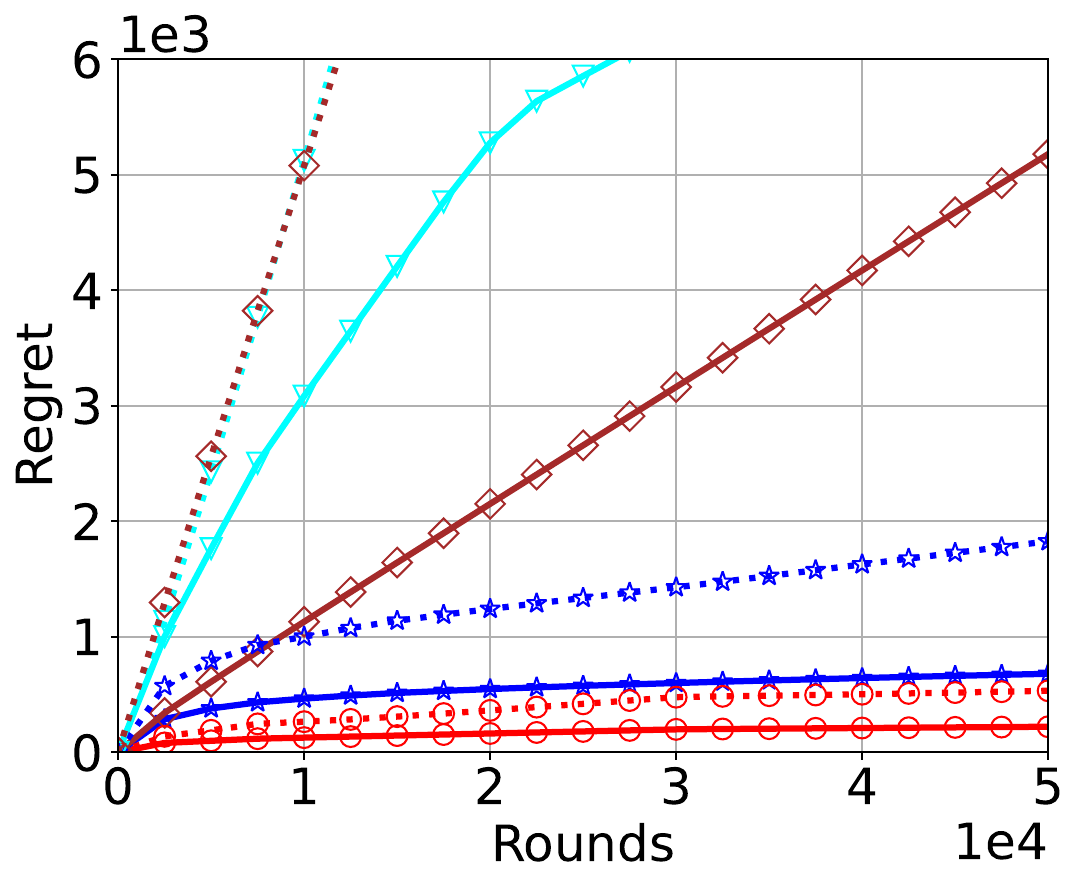} &  
        \includegraphics[width = 0.22\textwidth]{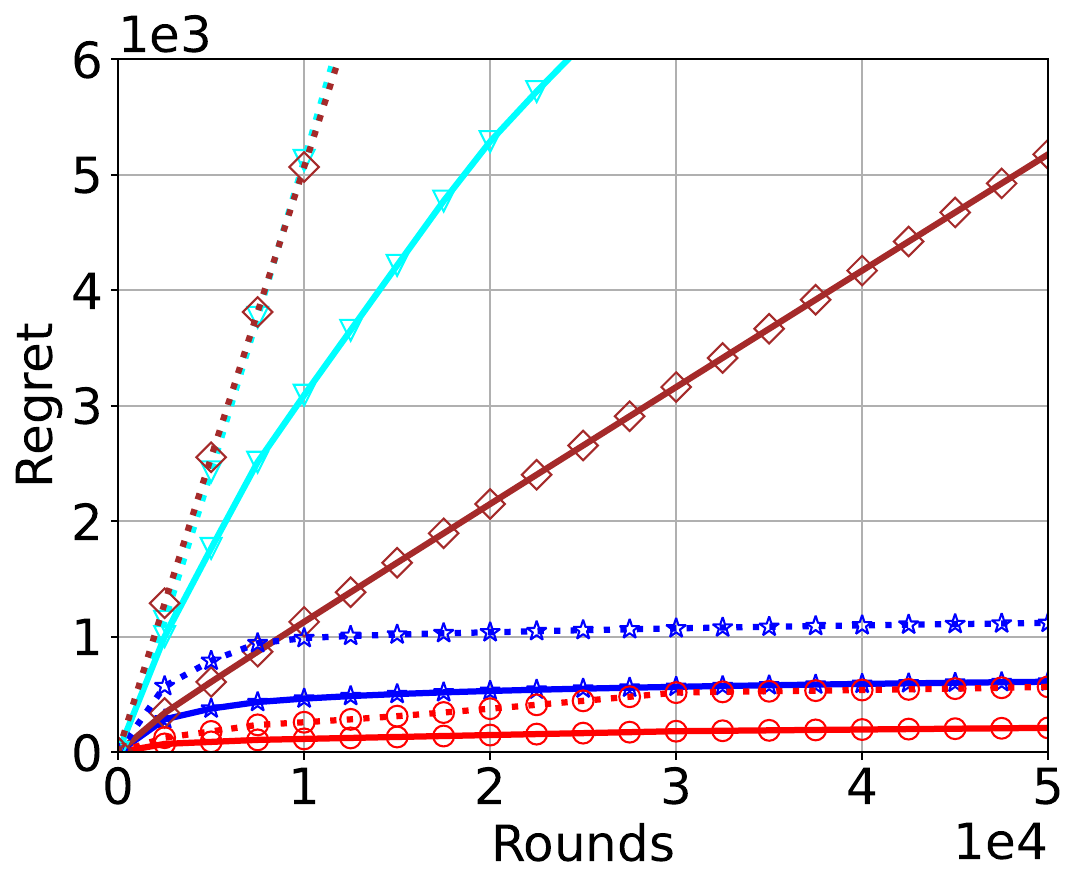} &
        \includegraphics[width = 0.22\textwidth]{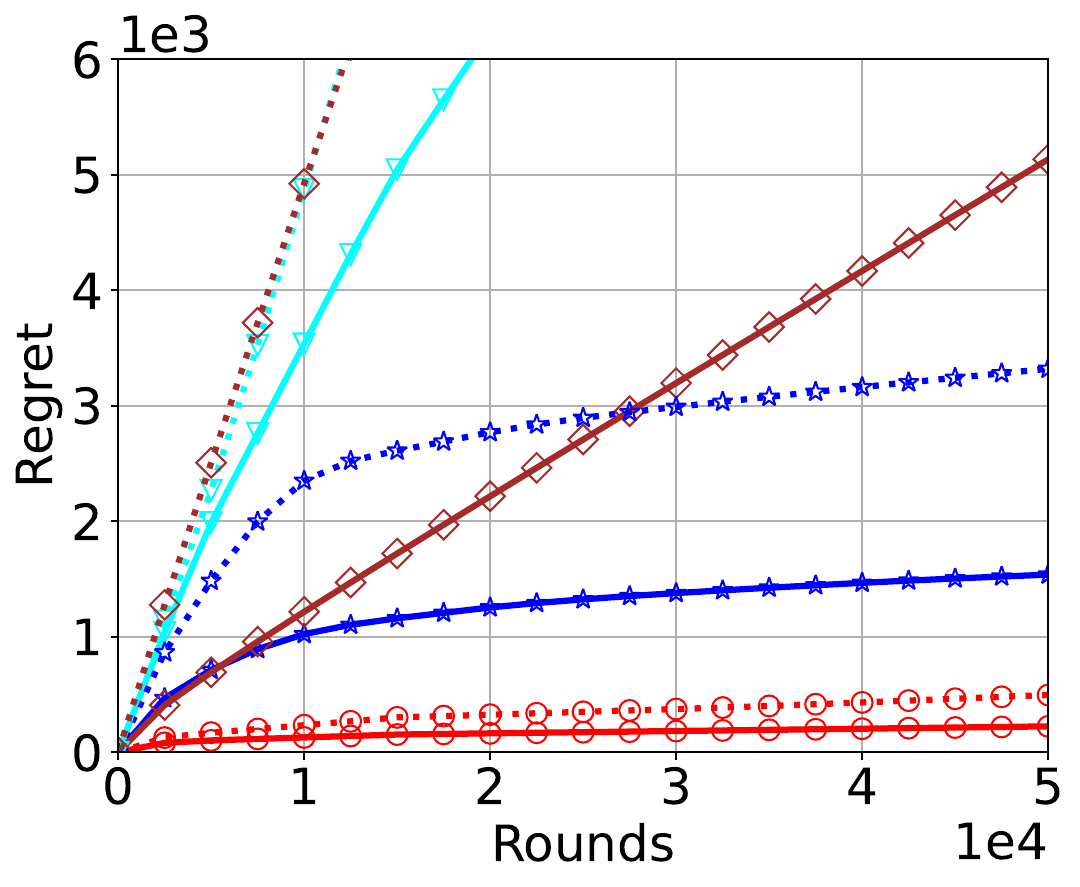} & 
        \includegraphics[width = 0.22\textwidth]{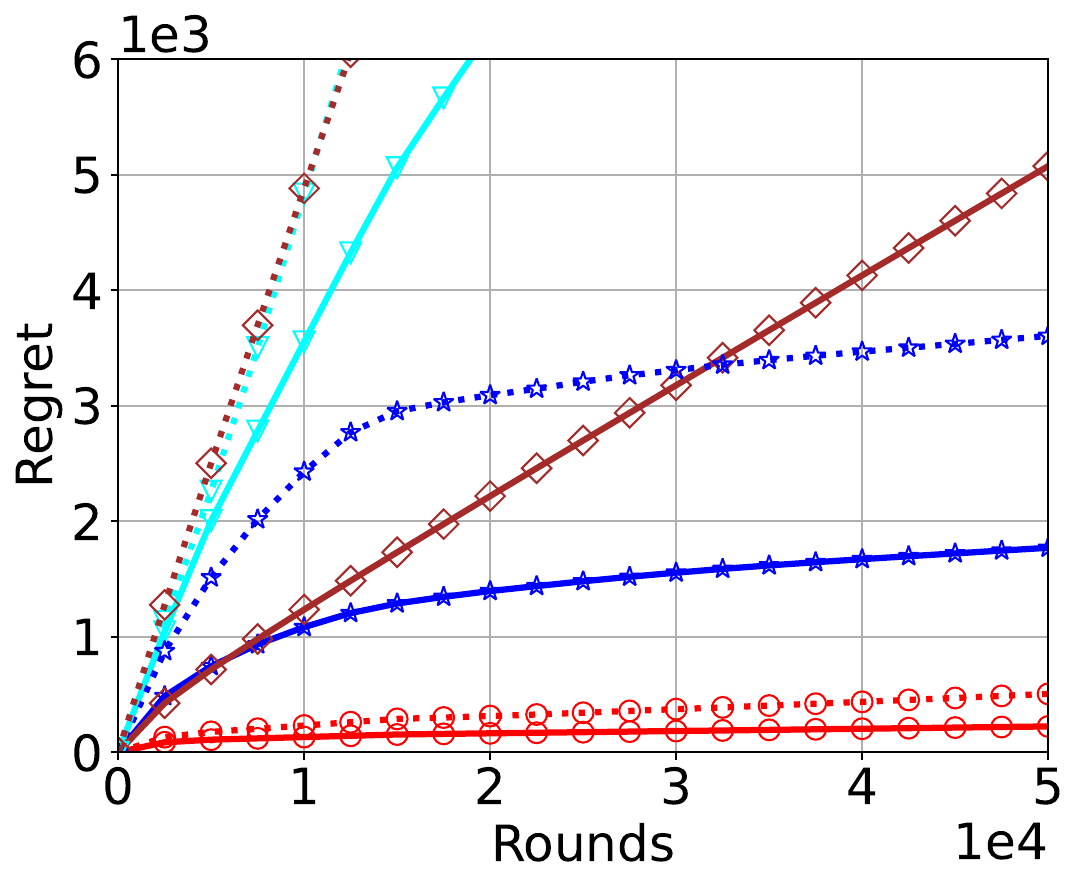}\\
        (a)  K = 10, C = 6000 &
        (b)  K = 10, C = 8000 &
        (c)  K = 20, C = 12000 &
        (d)  K = 20, C = 16000 \\
        \multicolumn{4}{c}{\textbf{(ii) The adversary attacks two agents (black nodes in Figure~\ref{fig:gra}).}} \\
    \end{tabular}
    \caption{DeMABAR vs. Resilient Decentralized UCB, IND-BARBAR,  and IND-FTRL in DeCMA2B under adversarial corruption.}
    \label{fig:cor-dec}
\end{figure*}

\begin{figure*}[!ht]
    \centering
    \begin{tabular}{cccc}
           \includegraphics[width = 0.22\textwidth]{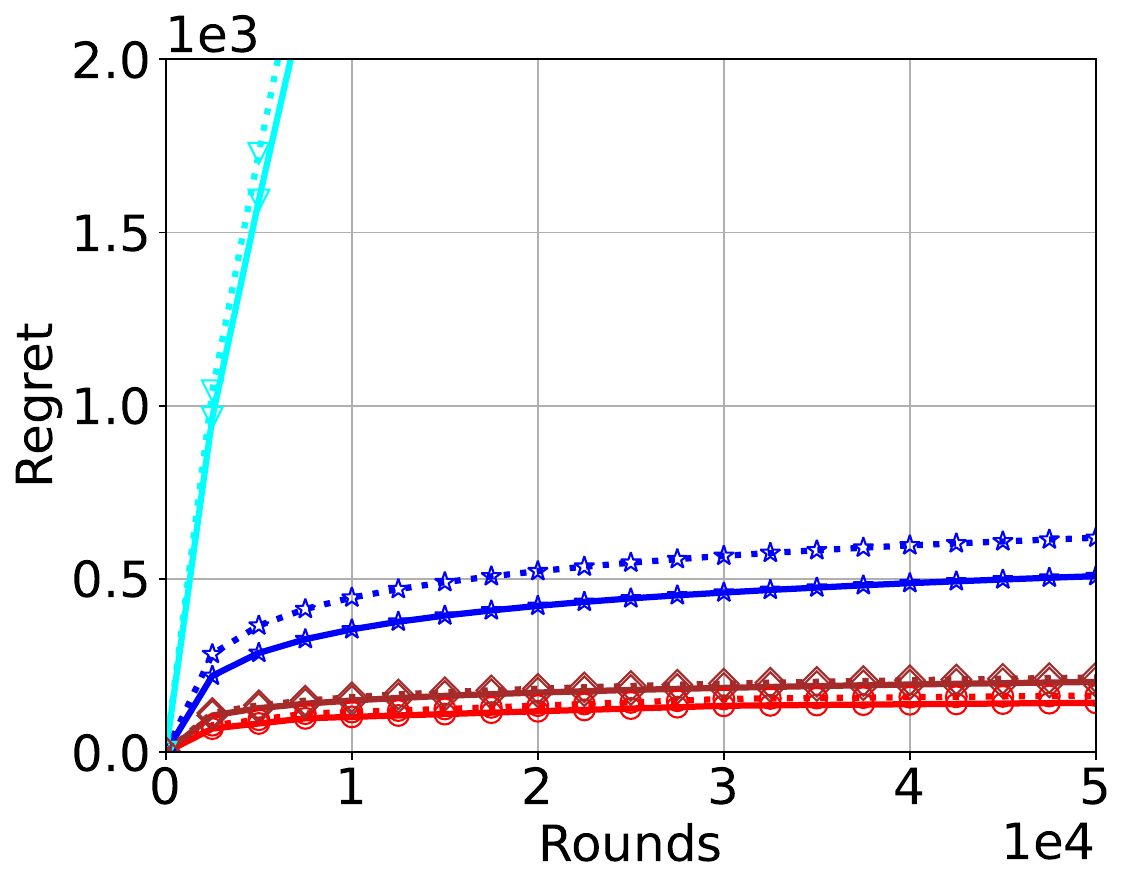} &  \includegraphics[width = 0.22\textwidth]{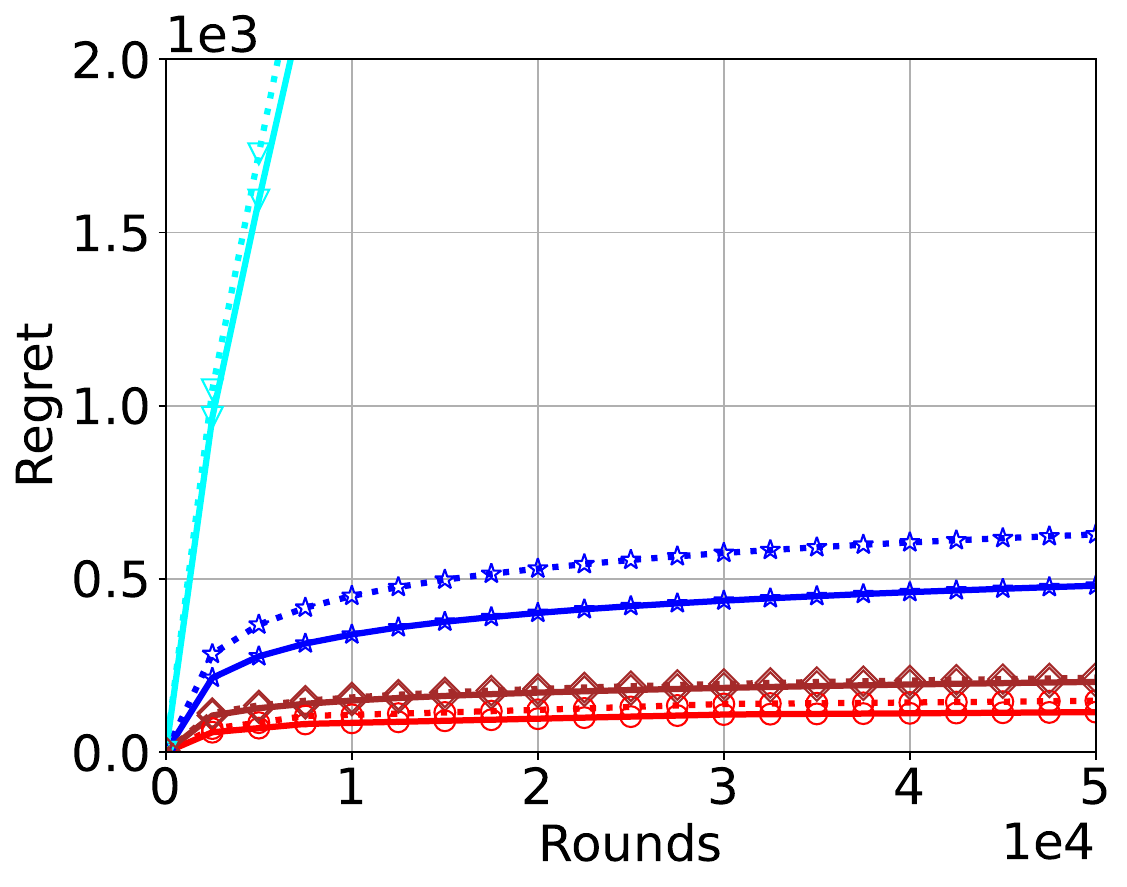} &
           \includegraphics[width = 0.22\textwidth]{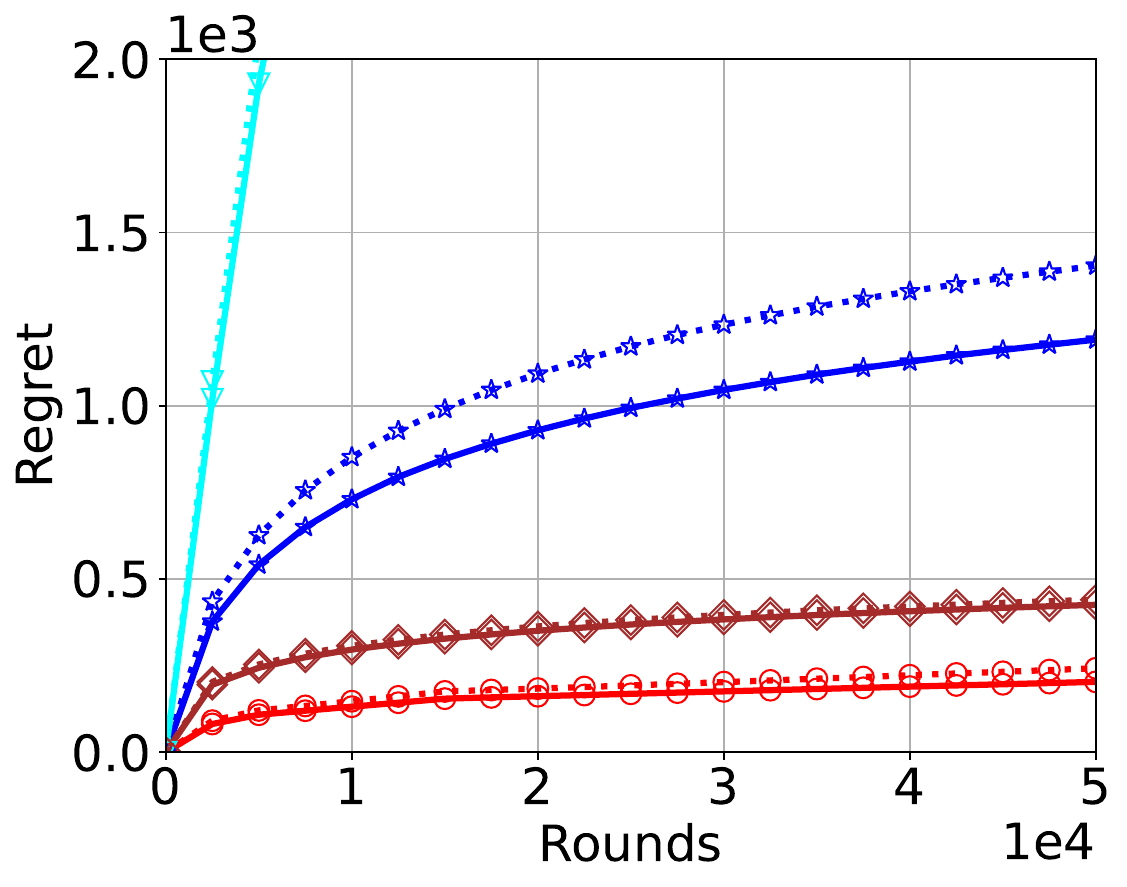} &  \includegraphics[width = 0.22\textwidth]{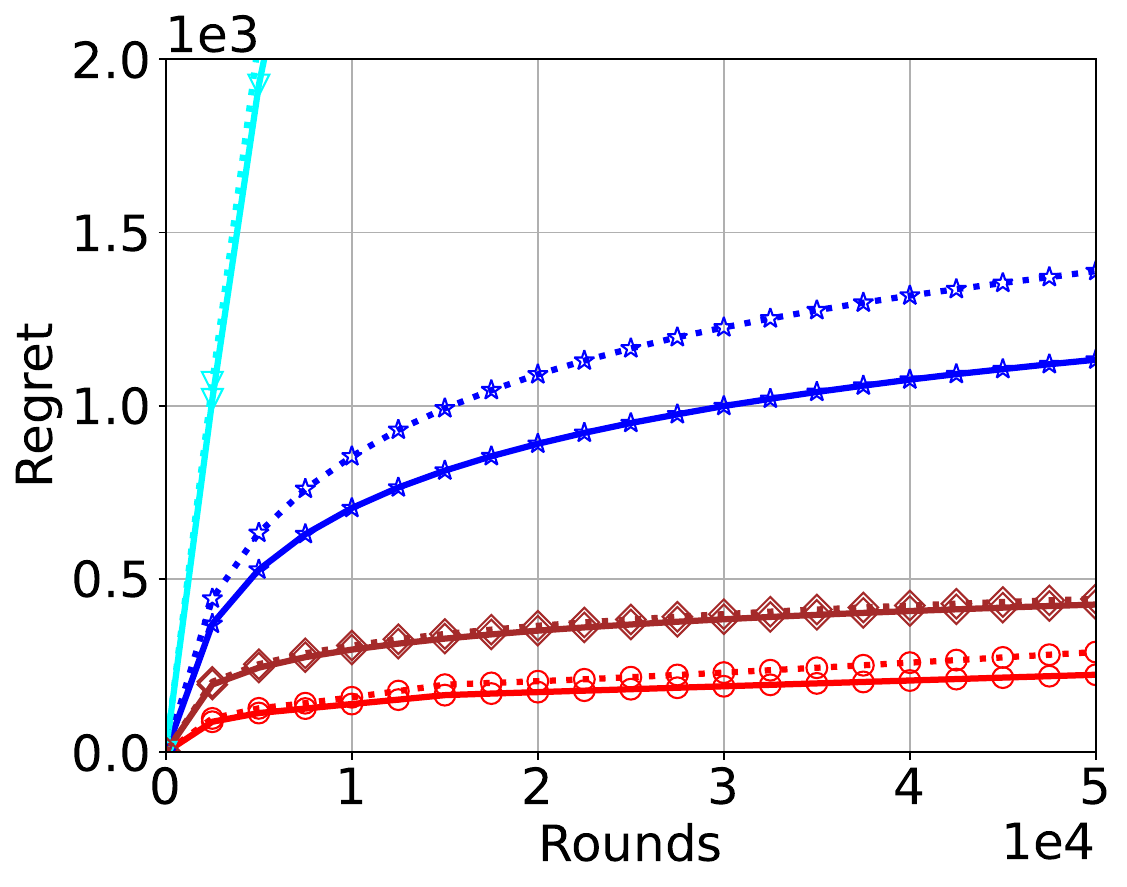} \\
        (a) K = 10, Adaptive Attack &
        (b) K = 10, Gaussian Attack &
        (c) K = 20, Adaptive Attack &
        (d) K = 20, Gaussian Attack
    \end{tabular}
    \caption{DeMABAR-F vs. Resilient Decentralized UCB, IND-BARBAR, and IND-FTRL in DeCMA2B with Byzantine agents. The Byzantine agents are the black nodes in Figure~\ref{fig:gra}.}
    \label{fig:byz}
\end{figure*}

In this section, we present numerical results to demonstrate the robustness and effectiveness of our algorithms. We consider the following baseline methods: IND-BARBAR~\cite{gupta2019better}, IND-FTRL~\cite{zimmert2021tsallis}, Resilient Decentralized UCB~\cite{zhu2023byzantine}, and DRAA~\cite{ghaffari2024multi}. Here, IND-BARBAR and IND-FTRL serve as non-cooperative baselines, wherein each agent runs the respective algorithm locally without inter-agent communication. We do not compare with \cite{liu2021cooperative} since the baseline DRAA is an improved version of their method. All experiments are implemented in Python 3.11 and conducted on a Windows laptop equipped with 16\,GB of memory and a single core of an Intel i7-13700H processor. For all experiments, we run $50$ independent trials and report the average total cumulative regret across all agents.

\subsection{Centralized CMA2B with Adversarial Corruption}
We first consider the distributed (centralized) setting and suppose the adversary can attack all agents. Each arm \(k \in \mathcal{K}\) has i.i.d.\ Gaussian rewards with mean \(\mu_k \sim U(0.1, 0.9)\) and standard deviation $0.01$. We set \(T=50{,}000\). Following \cite{lu2021stochastic}, arms with \(\mu_k \le 0.5\) are set as target arms. The adversary's goal is to make the agents pull the target arms as much as possible. Whenever an agent \(i\) pulls arm \(k\) with \(\mu_k>0.5\) and the realized reward is \(1\), the adversary corrupts this reward to \(\widetilde{r}_{i,k}(t)=0\) until the total corruption budget is exhausted.
We evaluate the performance under adversarial corruption in two scenarios:
\begin{itemize}
    \item The adversary corrupts all agents. The budgets are $C = 1500$ and $2000$ for $K=10$, and $C = 3000$ and $4000$ for $K=20$.
    \item The adversary targets $3$ out of $10$ agents (within $\beta = 0.3 < \alpha$), with a stronger attack using a higher budget per agent. The budgets are $C = 6000$ and $8000$ for $K=10$, and $C = 12000$ and $16000$ for $K=20$.
\end{itemize}

We present the experimental results in Figure~\ref{fig:cor-dis}, which show that our DeMABAR algorithm significantly outperforms all baseline methods, reflecting the advantage of collaboration and verifying our theoretical analysis. Notice that the Resilient Decentralized UCB method suffers nearly linear regret, which is consistent with previous findings~\cite{jun2018adversarial} that even small adversarial attacks can degrade UCB-family algorithms to linear regret.

\subsection{DeCMA2B with Adversarial Corruption}
We now switch to the decentralized scenario, using the network depicted in Figure~\ref{fig:gra}. The way to generate rewards and the adversary’s policy are the same as in the centralized environment. Given that DRAA and MA-BARBAT are not appropriate for the decentralized setting, we do not include them in this experiment. We present the numerical results in Figure~\ref{fig:cor-dec}, which show that our DeMABAR algorithm also outperforms all baseline methods. Notice that Resilient Decentralized UCB still suffers linear regret in this setting because the adversary can attack all agents. 
Meanwhile, since non-cooperative algorithms such as IND-FTRL and IND-BARBAR do not exploit inter-agent communication, the performance variations are minimal in centralized and decentralized settings.

\subsection{DeCMA2B with Byzantine Agents}
Finally, we consider the Byzantine decentralized setting. As shown in Figure~\ref{fig:gra}, two agents are Byzantine and each normal agent has at most one Byzantine neighbor. The way to generate rewards remains unchanged. Following~\cite{zhu2023byzantine}, we model two types of Byzantine attacks:
\begin{itemize}
    \item \textbf{Adaptive attack.} A Byzantine agent has full knowledge of the system and broadcasts misleading or opposite information. For example, if \(\mu_k=0.9\), it reports \(\mu_k=0.1\) to its neighbors. It also inflates the reported sample count \(n_{i,k}(t)\).
    \item \textbf{Gaussian attack.} Each Byzantine agent \(i\) picks a random bias \(\beta_{i,k}\in (0,1)\) for each arm~\(k\). For neighbor \(j\) at time \(t\), it draws \(c_{j,k}(t)\sim \cN\bigl(\beta_{i,k},0.001\bigr)\) and adds it to the relevant statistics (e.g., \(S_{j,k}^m/\tilde n_{j,k}^m\)) before transmitting.
\end{itemize}

Figure~\ref{fig:byz} illustrates that our DeMABAR algorithm outperforms all baselines under both attack models. Interestingly, the performance of IND-FTRL is comparable to that of DeMABAR. We believe this is because IND-FTRL is a non-cooperative algorithm, which means that Byzantine agents cannot impact the normal agents. Additionally, the network used in the experiment may not be large enough to demonstrate the advantages of collaboration. We anticipate that in larger networks, DeMABAR will significantly outperform any non-cooperative approaches by more effectively leveraging inter-agent collaboration and filtering out Byzantine agents.

\section{Conclusion}
\label{sec:con}
In this work, we present a novel robust algorithm for DeCMA2B, called DeMABAR, which facilitates effective collaboration among agents while remaining robust to both adversarial corruption and Byzantine attacks. The key idea of DeMABAR is to introduce a novel filtering mechanism to further diminish the influence of corruption. Notably, when the adversary can compromise only a small subset of agents, DeMABAR can be almost entirely unaffected by corruption. Our empirical evaluations align with these theoretical insights, showing that DeMABAR consistently outperforms baseline algorithms under adversarial corruption and in Byzantine environments.

\section{Acknowledgement}
Cheng Chen is supported by National
Natural Science Foundation of China (No. 62306116).

\bibliography{aaai2026}

\onecolumn
\appendix

\section{Auxiliary Lemmas}
\label{ape:auxiliarylemmas}
\begin{lemma}
\label{lem:epoch number} 
For Algorithm~\ref{algs:DeMABAR} with time horizon $T$, the number of epochs $M$ is at most $\ln(VT)$.
\end{lemma}
\begin{proof}
Since $(1-2\alpha)v_{\min}^w \leq V$ and $\lambda \geq 1$, the epoch $m$'s length satisfies
\[N_m = \left\lceil \frac{K\lambda 2^{2(m-1)}}{(1-2\alpha)v_{\min}^w} \right\rceil \geq \left\lceil \frac{K\lambda 2^{2(m-1)}}{V} \right\rceil \geq \frac{K2^{2(m-1)}}{V}.\]
From the lower bound of $N_m$ we can complete the proof.
\end{proof}

\begin{lemma}
\label{lem:k_i^m} 
For any agent $i \in [V]$, the length $N_m$ of epoch $m$ satisfies $N_m \geq \sum_{k\in[K]}n_{i,k}^m$, and $\tilde n_{i,k}^m \geq n_{i,k}^m$ for all $k \in [K]$.
\end{lemma}
\begin{proof}
    Since $\Delta_{i,k}^m = \max_{k \in [K]}\{2^{-m}, r_{i,*}^m - r_{i,k}^m\} \geq 2^{-m}$ and $v_{\min}^w = \min_{j \in [V]} |\cN_w(j)|$, we can get
    \[\sum_{k\in[K]} n_{i,k}^m = \sum_{k\in[K]} \frac{\lambda(\Delta_{i,k}^{m-1})^{-2}}{(1-2\alpha)|\cN_w(i)|} \leq \sum_{k\in[K]} \frac{\lambda2^{2(m-1)}}{(1-2\alpha)v_{\min}^w} \leq  \left\lceil \frac{K\lambda 2^{2(m-1)}}{(1-2\alpha)v_{\min}^w} \right\rceil = N_m.\]
    Last, because $v_{i}^w = \min\limits_{j \in \cN_w(i)}|\cN_w(j)| \leq |\cN_w(i)|$, we obtain the following inequality for all $k \in [K]$:
    \[n_{i,k}^m = \frac{\lambda(\Delta_{i,k}^{m-1})^{-2}}{(1-2\alpha)|\cN_w(i)|} \leq \frac{\lambda(\Delta_{i,k}^{m-1})^{-2}}{(1-2\alpha)v_i^w} \leq \tilde n_{i,k}^m.\]
    The proof is complete.
\end{proof}

\section{Stochastic Bandits with Adversarial Corruptions}

\subsection{The case of $\beta \leq \alpha$}
\label{ape:case1}

First, we define an event $\cL'$ as follows:
\begin{equation*}
    \cL' \triangleq \left\{ \forall\ i,\ k,\ m:\quad n_{i,k}^m \leq \tilde n_{j,k}^m \quad \textit{for all} \; j \in \cN_w(i)\right\}.
\end{equation*}
This means that after removing from $\cA_{i,k}^m$ all agents $j$ that satisfy $n_{i,k}^m > \tilde n_{j,k}^m$, the set $\cA_{i,k}^m$ equals $\cN_w(i)$. When the event $\cL'$ occurs, the event $\cL_{i,k}^m$ will never happen for all $i,k,m$.
\begin{lemma}
\label{lem:event l probability}
     The event $\cL'$ holds with probability at least $1 - \frac{K\ln(VT)}{VT^2}$.
\end{lemma}
The proof will be discussed later.

\begin{lemma}
\label{lem:experiment reward bound case1}
    If the event $\cL'$ occurs, for any fixed $i, k, m$, Algorithm \ref{algs:DeMABAR} satisfies
    \[\Pr\left[\,|r_{i,k}^m - \mu_k| \geq \sqrt{\frac{8\ln(2VT)}{(1-2\alpha)|\cN_w(i)| n_{i,k}^m}}\,\right] \leq \frac{1}{V^2T^2}.\]
\end{lemma}
\begin{proof}
    During each epoch $m$, agent $i$ pulls arm $k$ with probability $p_i^m(k) = \tilde n_{i,k}^m / N_m$. Consider $Y_{i,k}^t$, an indicator variable that determines whether agent $i$ pulls arm $k$. Define the corruption at step $t$ for agent $i$ on arm $k$ as $c_{i,k}^t := \widetilde{r}_{i,t}(k) - r_{i,t}(k)$. Let $E_m := [T_{m-1} + 1, \ldots, T_m]$ represent the $N_m$ time-steps constituting epoch $m$.
    
    Now, we explain why the corruption level $C$ can be completely removed in this case. For agent $i$, let the set of corrupted agents in $i$'s communication domain be denoted by $\cC_i^m$. According to Algorithm \ref{alg:filter}, if $\gamma$ corrupted agents are retained, where $\gamma = |\cB_{i,k}^m \cap \cC_i^m|$, this implies that there are at least $\frac{|\cA_{i,k}^m| - (1-2\alpha)|\cN_w(i)|}{2} - (\alpha - \gamma)|\cN_w(i)| = \alpha |\cN_w(i)|$ uncorrupted agents who are excluded due to having a smaller $\frac{S_{j,k}^m}{\tilde n_{j,k}^m}$ and at least $\alpha |\cN_w(i)|$ uncorrupted agents who are excluded due to having a larger $\frac{S_{j,k}^m}{\tilde n_{j,k}^m}$. In this context, for any agent $j$ satisfying $j \in \cB_{i,k}^m \cap \cC_i^m$, there exists a pair of distinct uncorrupted agents $j^-$ and $j^+$ who are filtered out, such that
    \[\frac{S_{j^-,k}^m}{\tilde n_{j^-,k}^m} \leq \frac{S_{j,k}^m}{\tilde n_{j,k}^m} \leq \frac{S_{j^+,k}^m}{\tilde n_{j^+,k}^m}\]
    Thus, $\frac{S_{j,k}^m}{\tilde n_{j,k}^m}$ can be represented as a convex combination of $\frac{S_{j^-,k}^m}{\tilde n_{j^-,k}^m}$ and $\frac{S_{j^+,k}^m}{\tilde n_{j^+,k}^m}$, as follows:
    \begin{align}
        \frac{S_{j,k}^m}{\tilde n_{j,k}^m} = \theta_j \frac{S_{j^-,k}^m}{\tilde n_{j^-,k}^m} + (1-\theta_j) \frac{S_{j^+,k}^m}{\tilde n_{j^+,k}^m}, \quad \theta_j \in [0, 1].\label{eq:totallossdecompositioncase1}
    \end{align}
    Recalling the definition of $r_{i,k}^m$ and noting that $|\cB_{i,k}^m| \geq (1-2\alpha)|\cN_w(i)|$, we have:
    \[r_{i,k}^m = \min \left\{\frac{1}{|\cB_{i,k}^m|}\sum_{j \in \cB_{i,k}^m}\frac{S_{j,k}^m}{\tilde n_{j,k}^m}, 1\right\} \leq \frac{1}{(1-2\alpha)|\cN_w(i)|}\sum_{j \in \cB_{i,k}^m}\frac{S_{j,k}^m}{\tilde n_{j,k}^m}.\]
    The quantity we intend to control is then represented as:
    \begin{align*}
        r_{i,k}^m &\leq \frac{1}{(1-2\alpha)|\cN_w(i)|}\sum_{j \in \cB_{i,k}^m}\frac{S_{j,k}^m}{\tilde n_{j,k}^m}
        = \frac{1}{(1-2\alpha)|\cN_w(i)|} \sum_{j \in \cA_{i,k}^m}w_j \frac{S_{j,k}^m}{\tilde n_{j,k}^m} \quad  \left(w_j \in [0,1], \; \sum_{j \in \cA_{i,k}^m}w_j = (1-2\alpha)|\cN_w(i)|\right)\\
        &= \frac{1}{(1-2\alpha)|\cN_w(i)|} \sum_{j \in \cA_{i,k}^m} \sum_{t \in E_m}w_j \frac{Y_{j,k}^t r_{j,k}^t}{\tilde n_{j,k}^m} 
        = \frac{1}{(1-2\alpha)|\cN_w(i)|n_{i,k}^m} \sum_{j \in \cA_{i,k}^m} \sum_{t \in E_m}\frac{w_j Y_{j,k}^t r_{j,k}^tn_{i,k}^m}{\tilde n_{j,k}^m} 
    \end{align*}
    where the first equality holds because we decompose $\frac{S_{j,k}^m}{\tilde n_{j,k}^m}$ by (\ref{eq:totallossdecompositioncase1}).
    To simplify the analysis, we focus on the following component:
    \[A_{i,k}^m = \sum_{j \in \cA_{i,k}^m} \sum_{t \in E_m}\frac{w_j Y_{j,k}^t r_{j,k}^tn_{i,k}^m}{\tilde n_{j,k}^m}.\]
    Notice that $r_{j,k}^t$ is independently drawn from an unknown distribution with mean $\mu_k$, and $Y_{j,k}^t$ is independently drawn from a Bernoulli distribution with mean $q_{j,k}^m := \tilde n_{j,k}^m / N_m$. Since $\tilde n_{j,k}^m \geq n_{i,k}^m$, we have
    \begin{align*}
        \forall \; j,k,m,t:\quad \frac{w_j Y_{j,k}^t r_{j,k}^tn_{i,k}^m}{\tilde n_{j,k}^m} \leq 1.
    \end{align*}
    Furthermore, we can obtain
    \begin{align*}
        \BE[A_{i,k}^m] 
        = \sum_{j \in \cA_{i,k}^m} \sum_{t \in E_m}\frac{w_j Y_{j,k}^t r_{j,k}^tn_{i,k}^m}{\tilde n_{j,k}^m} 
        = \sum_{j \in \cN_w(i)} w_j n_{i,k}^m \mu_k 
        = (1-2\alpha) |\cN_w(i)| n_{i,k}^m \mu_k.
    \end{align*}    
    Therefore, by utilizing the Chernoff-Hoeffding inequality (Theorem 1.1 in \cite{dubhashi2009concentration}), we derive the following result:
    \[\Pr\left[\,\left|A_{i,k}^m - (1-2\alpha) |\cN_w(i)| n_{i,k}^m \mu_k\right| \geq \sqrt{3(1-2\alpha) |\cN_w(i)| n_{i,k}^m \mu_k\ln(4V^2T^2)}\,\right] \leq \frac{1}{2V^2T^2}.\]
    Through simple calculations, we can get
    \begin{equation}
    \Pr\left[\,\left|\frac{A_{i,k}^m}{(1-2\alpha)|\cN_w(i)| n_{i,k}^m} - \mu_k\right| \geq \sqrt{\frac{6\ln(2VT)}{(1-2\alpha)|\cN_w(i)| n_{i,k}^m}}\,\right] \leq \frac{1}{2V^2T^2}.\label{eq:partA case1}
    \end{equation}
    The proof is complete.
\end{proof}
We define an event $\cE$ for epoch $m$ as follows:
\begin{equation*}
    \cE \triangleq \left\{ \forall\ i,\ k ,\ m: |r_{i,k}^m - \mu_k| \leq \sqrt{\frac{8\ln(2VT)}{(1-2\alpha)|\cN_w(i)|n_{i,k}^m}}\right\}.
\end{equation*}
Then we can establish a lower bound on the probability of the event $\cE$ occurring by the following lemma.
\begin{lemma}
\label{lem:event e probability case1}
     The event $\cE$ holds with probability at least $1 - \frac{K\ln(VT)}{VT^2}$.
\end{lemma}
\begin{proof}
    By Lemma \ref{lem:experiment reward bound case1}, we can get the following inequality for any $i, k$ and $m$:
    \[\Pr\left[\,|r_{i,k}^m - \mu_k| \geq \sqrt{\frac{8\ln(2VT)}{(1-2\alpha)|\cN_w(i)|n_{i,k}^m}}\,\right] \leq \frac{1}{V^2T^2}.\]
    A union bound over the $K$ arms, $V$ agents, and at most $\ln(VT)$ epochs indicates that the success probability of event $\cE$ is at least $1 - \frac{K\ln(VT)}{VT^2}$.
\end{proof}
Our discussion below will be based on the occurrence of event $\cE$.
\begin{lemma}
\label{lem:suboptimality-gap bound case1}
    For any fixed $i, k$ and $m$, it follows that
    \[\frac{6}{7}\Delta_k - \frac{3}{4}2^{-m} \leq \Delta_{i,k}^{m} \leq \frac{8}{7}\Delta_k + 2^{-m}.\]
\end{lemma}
\begin{proof}
    First, we have
    \[\sqrt{\frac{8\ln(2VT)}{(1-2\alpha)|\cN_w(i)|n_{i,k}^m}} = \sqrt{\frac{8\ln(2VT)}{2^9\ln(2VT)(\Delta_{i,k}^{m-1})^{-2}}} = \frac{\Delta_{i,k}^{m-1}}{8},\]
    Therefore, we can get
    \[- \frac{1}{8}\Delta_{i,k}^{m-1} \leq r_{i,k}^m - \mu_{k} \leq \frac{1}{8}\Delta_{i,k}^{m-1}.\]
    Additionally, given that
    \[r_{i,*}^m \leq \max_{k\in[K]} \left\{\mu_{k} + \frac{1}{8} \Delta_{i,k}^{m-1} - \frac{1}{8} \Delta_{i,k}^{m-1}\right\} \leq \mu_{k^*},\]
    \[r_{i,*}^m = \max_{k\in[K]} \left\{r_{i,k}^m - \frac{1}{8} \Delta_{i,k}^{m-1}\right\} \geq r_{i,k^*}^m - \frac{1}{8}\Delta_{i,k^*}^{m-1} \geq \mu_{k^*} - \frac{1}{4}\Delta_{i,k^*}^{m-1},\]
    it follows that
    \[- \frac{\Delta_{i,k^*}^{m-1}}{4} \leq r_{i,*}^m - \mu_{k^*} \leq 0.\]
    We now establish the upper bound for $\Delta_{i,k}^m$ using induction on epoch $m$. \\
    For the base case $m = 1$, the statement is trivial as $\Delta_{i,k}^0 = 1$ for all $k \in [K]$. \\
    Assuming the statement is true for $m-1$, we then have
    \begin{equation*}
    \begin{split}
        \Delta_{i,k}^m &= r_{i,*}^m - r_{i,k}^m
        = (r_{i,*}^m - \mu_{k^*}) + (\mu_{k^*} - \mu_k) + (\mu_k - r_{i,k}^m) \\
        &\leq \Delta_k + \frac{1}{8}\Delta_{i,k}^{m-1}
        \leq \Delta_k + \frac{1}{8}\left(\frac{8 \Delta_k}{7} + 2^{-(m-1)}\right)
        \leq \frac{8 \Delta_k}{7} + 2^{-m},
    \end{split}
    \end{equation*}
    where the second inequality follows from the induction hypothesis. \\
    Next, we establish the lower bound for $\Delta_{i,k}^m$. Specifically, we demonstrate that
    \begin{equation*}
    \begin{split}
        \Delta_{i,k}^m &=  r_{i,*}^m - r_{i,k}^m
        = (r_{i,*}^m - \mu_{k^*}) + (\mu_{k^*} - \mu_k) + (\mu_k - r_{i,k}^m) \\
        &\geq  - \frac{1}{4}\Delta_{i,k^*}^{m-1} + \Delta_k - \frac{1}{8}\Delta_{i,k}^{m-1}
        \geq \Delta_k - \frac{1}{8}\left(\frac{8 \Delta_{k}}{7} + 2^{-(m-1)}\right) - \frac{1}{4}\left(\frac{8 \Delta_{k^*}}{7} + 2^{-(m-1)}\right)
        \geq \frac{6}{7}\Delta_k - \frac{3}{4} 2^{-m}.
    \end{split}
    \end{equation*}
    where the third inequality comes from the upper bound of $\Delta_{i,k}^{m-1}$.
\end{proof}

\begin{lemma}
\label{lem:estimated gap ratio}
    For any fixed $k$, $m$, and two agents $i, j$, it follows that
    \[\frac{\Delta_{i,k}^{m}}{\Delta_{j,k}^{m}} \in \left[\frac{1}{4}, 4\right].\]
\end{lemma}
\begin{proof}
    Since $\Delta_{i,k}^m = \max_k\{2^{-m}, r_{i,*}^m - r_{i,k}^m\} \geq 2^{-m}, \forall i \in [V]$, which means that $\Delta_{i,k}^m \geq 2^{-m}$ and $\Delta_{j,k}^m \geq 2^{-m}$. By Lemma \ref{lem:suboptimality-gap bound case1}, when $\Delta_k \leq \frac{49}{24}2^{-m}$, then
    \[\frac{6}{7}\Delta_k - \frac{3}{4}2^{-m} \leq 2^{-m}.\]
    Hence, we have
    \[\frac{\Delta_{i,k}^m}{\Delta_{j,k}^m}  \leq \frac{\frac{8}{7}\Delta_k + 2^{-m}}{2^{-m}} = 1 + \frac{8\Delta_k}{7\cdot 2^{-m}} < 4.\]
    When $\Delta_k \geq \frac{49}{24}2^{-m}$, then
    \[\frac{\Delta_{i,k}^m}{\Delta_{j,k}^m} \leq \frac{\frac{8}{7}\Delta_k + 2^{-m}}{\frac{6}{7}\Delta_k - \frac{3}{4}2^{-m}} = \frac{4}{3} + \frac{2\cdot 2^{-m}}{\frac{6}{7}\Delta_k - \frac{3}{4}2^{-m}} < 4.\]
    The inequality reaches its maximum value when $\Delta_k = \frac{49}{24}2^{-m}$.
    Because $i$ and $j$ are equivalent, the proof can be completed by swapping their positions.
\end{proof}

Since for any agent $j \in \cN_w(i)$, we have
\[\tilde n_{j,k}^m = \min \left\{\lambda 2^{2(m-1)}, \frac{16\lambda (\Delta_{j,k}^{m-1})^{-2}}{(1-2\alpha)v_j^w} \right\} \geq \min \left\{\lambda 2^{2(m-1)}, \frac{16\lambda (\Delta_{i,k}^{m-1})^{-2} (\Delta_{j,k}^{m-1}/ \Delta_{i,k}^{m-1})^{-2}}{(1-2\alpha)|\cN_w(i)|} \right\} \geq n_{i,k}^m.\]
So we can say that when event $\cE$ occurs, event $\cL$ must occur, and Lemma~\ref{lem:event l probability} is complete. \\
Next we will bound the regret and partition the proof into two cases. In each epoch $m$, for any arm $k \neq k_i^m$, we have
\[\tilde n_{i,k}^m = \min \left\{\lambda 2^{2(m-1)}, \frac{16\lambda (\Delta_{i,k}^{m-1})^{-2}}{(1-2\alpha)v_i^w} \right\},\]
and for arm $k_i^m$ we have
\[\tilde n_{i,k_i^m}^m = N_m - \sum_{k \neq k_i^m} \tilde n_{i,k}^m < N_m.\]
\paragraph{Case 1:} $\Delta_k \leq 3 \cdot 2^{-m}$. \\
Since $\Delta_{i,k}^{m-1} = \max \{2^{-(m-1)}, r_{i,*}^{m-1} - r_{i,k}^{m-1}\}$, we have
\[\forall \ i: \quad \Delta_{i,k}^{m-1} \geq 2^{-(m-1)} \geq \frac{2\Delta_k}{3}.\]
Therefore, we can get the following inequality for all arms $k \neq k_i^m$:
\[\tilde n_{i,k}^m = \min \left\{\lambda 2^{2(m-1)}, \frac{16\lambda (\Delta_{i,k}^{m-1})^{-2}}{(1-2\alpha)v_i^w} \right\} \leq \frac{16\lambda (\Delta_{i,k}^{m-1})^{-2}}{(1-2\alpha)v_i^w} \leq \frac{36\lambda}{(1-2\alpha)v_i^w\Delta_k^2}.\]
For arm $k_i^m$, since $\Delta_{i,k_i^m}^{m-1} = 2^{-(m-1)}$, we have
\begin{align*}
    \tilde n_{i,k_i^m}^m < N_m = \left\lceil \frac{K\lambda 2^{2(m-1)}}{(1-2\alpha)v_{\min}^w} \right\rceil \leq
    \frac{K\lambda (\Delta_{k_i^m}^{m-1})^{-2}}{(1-2\alpha)v_{\min}^w} + 1 \leq \frac{9K\lambda}{4(1-2\alpha)v_{\min}^w\Delta_{k_i^m}^2} + 1 \leq \frac{9K\lambda}{4(1-2\alpha)v_{\min}^w\Delta^2} + 1.
\end{align*}
This epoch, which satisfies the given conditions $\Delta_k \leq 3 \cdot 2^{-m}$, is bounded by $\log(1/\Delta)$ which can be considered as a constant.

\paragraph{Case 2:} $\Delta_k > 3 \cdot 2^{-m}$. \\
In this case, by Lemma \ref{lem:suboptimality-gap bound case1} we have
\[\forall \ i: \quad \Delta_{i,k}^{m-1} \geq \frac{6}{7}\Delta_k - \frac{3}{4}2^{-m} \geq \Delta_k\left(\frac{6}{7} - \frac{1}{4}\right) \geq 0.61 \Delta_k.\]
In this case, it is impossible for $\Delta_{k_i^m} > 3 \cdot 2^{-m}$ to occur. Since $\Delta_{i,k_i^m}^{m-1} \geq 0.61 \Delta_{k_i^m} > 2^{-(m-1)}$, this does not align with the algorithm's selection criterion $\Delta_{i,k_i^m}^{m-1} = 2^{-(m-1)}$. Therefore, arm $k_i^m$ must be the optimal arm.\\
So we can obtain for all suboptimal arms
\begin{align*}
    \tilde n_{i,k}^m = \min \left\{\lambda 2^{2(m-1)}, \frac{16\lambda (\Delta_{i,k}^{m-1})^{-2}}{(1-2\alpha)v_i^w} \right\} \leq \frac{16\lambda (\Delta_{i,k}^{m-1})^{-2}}{(1-2\alpha)v_i^w} \leq \frac{16\lambda}{0.61^2(1-2\alpha)v_i^w\Delta_k^2}
    \leq \frac{43\lambda}{(1-2\alpha)v_i^w\Delta_k^2}.
\end{align*}
Based on the cases mentioned above, we have the following inequality:
\begin{align*}
        R_i(T) &\leq \sum_{m=1}^M \sum_{\Delta_k > 0} \Delta_{k} \tilde n_{i,k}^m + \sum_{m=1}^M (w - 1)\Delta_{k_i^m} + \frac{KT\ln(VT)}{VT^2}
        \leq \sum_{m=1}^M \sum_{\Delta_k > 0} \Delta_{k} \tilde n_{i,k}^m + (w-1)\ln(\Delta^{-1}) + \frac{K\ln(VT)}{VT}\\
        &\leq \sum_{m=1}^M \sum_{\Delta_k > 0} \left(\Delta_k \frac{43\lambda}{(1-2\alpha)v_i^w\Delta_k^2} + \BI(\Delta_k < 3\cdot 2^{-m})\left(\frac{9K\lambda}{4(1-2\alpha)v_{\min}^w\Delta^2} + 2\right)\right) + (w-1)\ln(\Delta^{-1}) \\
        &= O\left(\frac{\ln^2(VT)}{(1-2\alpha)v_i^w\Delta_k} + \frac{K\ln(VT)\ln(\Delta^{-1})}{(1-2\alpha)v_{\min}^w\Delta}\right).
\end{align*}
Since the agents only communicate at the end of each epoch, we have
\[\textrm{Cost}(T) = \sum_{i\in [V]}wM = wV\ln(VT).\]

\subsection{The case of $\beta > \alpha$}

\begin{lemma}
\label{lem:experiment reward bound case2}
    For any fixed $i, k, m$, regardless of whether the event $\cL_{i,k}^m$ occurs, Algorithm \ref{algs:DeMABAR} satisfies
    \[\Pr\left[\,|r_{i,k}^m - \mu_k| \geq \sqrt{\frac{8\ln(2VT)}{(1-2\alpha)|\cN_w(i)| n_{i,k}^m}} + \frac{2C_m}{(1-2\alpha)|\cN_w(i)|N_m}\, \right] \leq \frac{1}{V^2T^2}.\]
    It is worth noting that when the event $\cL_{i,k}^m$ occurs, $n_{i,k}^m = \min_{j \in \cN_w(i)} \tilde n_{j,k}^m$, otherwise $n_{i,k}^m = \frac{\lambda (\Delta_{i,k}^{m-1})^{-2}}{(1-2\alpha)|\cN_w(i)|}$.
\end{lemma}
\begin{proof}
    During each epoch $m$, agent $i$ pulls arm $k$ with probability $p_i^m(k) = \tilde n_{i,k}^m / N_m$. Consider $Y_{i,k}^t$, an indicator variable that determines whether agent $i$ pulls arm $k$. Define the corruption at step $t$ for agent $i$ on arm $k$ as $c_{i,k}^t := \widetilde{r}_{i,t}(k) - r_{i,t}(k)$. Let $E_m := [T_{m-1} + 1, \ldots, T_m]$ represent the $N_m$ time-steps constituting epoch $m$.
    
    Recalling the definition of $r_{i,k}^m$ and using the fact that $|\cB_{i,k}^m| \geq (1-2\alpha)|\cN_w(i)|$, we have
    \[r_{i,k}^m = \min \left\{\frac{1}{|\cB_{i,k}^m|}\sum_{j \in \cB_{i,k}^m}\frac{S_{j,k}^m}{\tilde n_{j,k}^m}, 1\right\} \leq \frac{1}{(1-2\alpha)|\cN_w(i)|}\sum_{j \in \cB_{i,k}^m}\frac{S_{j,k}^m}{\tilde n_{j,k}^m}.\]
    The quantity we intend to control is then represented as:
    \begin{align*}
        r_{i,k}^m &\leq \frac{1}{(1-2\alpha)|\cN_w(i)|}\sum_{j \in \cB_{i,k}^m}\frac{S_{j,k}^m}{\tilde n_{j,k}^m}
        = \frac{1}{(1-2\alpha)|\cN_w(i)|n_{i,k}^m}\sum_{j \in \cB_{i,k}^m}\sum_{t \in E_m}\frac{n_{i,k}^mY_{j,k}^t(r_{j,k}^t + c_{j,k}^t)}{\tilde n_{j,k}^m}.
    \end{align*}
    To simplify the analysis, we focus on the following two components:
    \[A_{i,k}^m = \sum_{j \in \cB_{i,k}^m}\sum_{t \in E_m}\frac{n_{i,k}^m Y_{j,k}^t r_{j,k}^t}{\tilde n_{j,k}^m}, \quad B_{i,k}^m = \sum_{j \in \cB_{i,k}^m}\sum_{t \in E_m}\frac{n_{i,k}^mY_{j,k}^t(r_{j,k}^t + c_{j,k}^t)}{\tilde n_{j,k}^m}\]
    Notice that $r_{j,k}^t$ is independently drawn from an unknown distribution with mean $\mu_k$, and $Y_{j,k}^t$ is independently drawn from a Bernoulli distribution with mean $q_{j,k}^m := \tilde n_{j,k}^m / N_m$. Since $\tilde n_{j,k}^m \geq n_{i,k}^m$, we have
    \begin{align*}
        \forall \; j,k,m,t:\quad \frac{Y_{j,k}^t r_{j,k}^tn_{i,k}^m}{\tilde n_{j,k}^m} \leq 1.
    \end{align*}
    Furthermore, we can obtain
    \begin{align*}
        \BE[A_{i,k}^m] 
        = \sum_{j \in \cB_{i,k}^m} \sum_{t \in E_m}\frac{ Y_{j,k}^t r_{j,k}^tn_{i,k}^m}{\tilde n_{j,k}^m} 
        = \sum_{j \in \cB_{i,k}^m}  n_{i,k}^m \mu_k 
        = (1-2\alpha) |\cN_w(i)| n_{i,k}^m \mu_k.
    \end{align*}    
    Therefore, by utilizing the Chernoff-Hoeffding inequality (Theorem 1.1 in \cite{dubhashi2009concentration}), we derive the following result:
    \[\Pr\left[\,\left|A_{i,k}^m - (1-2\alpha) |\cN_w(i)| n_{i,k}^m \mu_k\right| \geq \sqrt{3(1-2\alpha) |\cN_w(i)| n_{i,k}^m \mu_k\ln(4V^2T^2)}\,\right] \leq \frac{1}{2V^2T^2}.\]
    Through simple calculations, we can get
    \begin{equation}
        \Pr\left[\,\left|\frac{A_{i,k}^m}{(1-2\alpha)|\cN_w(i)| n_{i,k}^m} - \mu_k\right| \geq \sqrt{\frac{6\ln(2VT)}{(1-2\alpha)|\cN_w(i)| n_{i,k}^m}}\,\right] \leq \frac{1}{2V^2T^2}.\label{eq:partA case2}
    \end{equation}
    Next, we proceed to establish a bound on the deviation of $B_{i,k}^m$. To do this, we define a random sequence $X_i^1, \ldots, X_i^T$, where each term is given by $X_i^t = \frac{(Y_{j,k}^t - q_{j,k}^m) c_{j,k}^t n_{i,k}^m}{\tilde n_{j,k}^m}$ for all $t$ and for all $j \in \cB_{i,k}^m$. This sequence $\{X_i^t\}_{t=1}^T$ forms a martingale difference sequence with respect to the filtration $\{\cF_t\}_{t=1}^T$, which is generated by the historical information. Specifically, because the corruption $c_{j,k}^t$ becomes a deterministic value when conditioned on $\cF_{t-1}$ and because $\BE [Y_{j,k}^t | \cF_{t-1}] = q_{j,k}^m$, we can conclude:
    \[\BE[X_i^t | \cF_{t-1}] = \BE[Y_{j,k}^t - q_{j,k}^m | \cF_{t-1}]\frac{ c_{j,k}^t n_{i,k}^m}{\tilde n_{j,k}^m} = 0.\]
    Additionally, we have $|X_i^t| \leq 1$ for all $t$, and the predictable quadratic variation of this martingale can be bounded as follows:
    \begin{align*}
        \Var(X_i^t) &= \sum_{j \in \cB_{i,k}^{m}} \sum_{t\in E_m} \BE[(X_i^t)^2 | \cF_{t-1}] \leq \sum_{j \in \cB_{i,k}^{m}} \sum_{t\in E_m} \Var(Y_{j,k}^t) \frac{|c_{j,k}^t|n_{i,k}^m}{\tilde n_{j,k}^m}  \\
        &\leq \sum_{j \in \cB_{i,k}^{m}} \sum_{t\in E_m} \frac{q_{j,k}^m |c_{j,k}^t| n_{i,k}^m}{\tilde n_{j,k}^m} \leq \sum_{j \in \cB_{i,k}^{m}} \sum_{t\in E_m} \frac{|c_{j,k}^t|n_{i,k}^m}{N_m}.
    \end{align*}
    Applying the concentration inequality for martingales (Theorem 1 in \cite{beygelzimer2011contextual}), we find that, with probability at least $1 - \frac{1}{4V^2T^2}$,
    \begin{align*}
        B_{i,k}^m 
        &\leq \sum_{j \in \cB_{i,k}^{m}} \sum_{t\in E_m}
        \frac{q_{j,k}^m c_{j,k}^t n_{i,k}^m}{\tilde n_{j,k}^m}
        + \Var(X_i^t) + \ln(4V^2T^2) \\
        &\leq 2\sum_{j \in \cB_{i,k}^{m}} \sum_{t\in E_m} 
        \frac{|c_{j,k}^t| n_{i,k}^m}{N_m}
        + \ln(4V^2T^2). 
    \end{align*}
    Given that $\sum_{j \in \cB_{i,k}^{m}} \sum_{t\in E_m} |c_{j,k}^t| \leq C_m$, and $(1-2\alpha)|\cN_w(i)|n_{i,k}^m \geq \lambda \geq 16 \ln(4V^2T^2)$, it follows that with the same probability:
    \[\frac{B_{i,k}^m}{(1-2\alpha)|\cN_w(i)|n_{i,k}^m} \leq \sqrt{\frac{\ln(4V^2T^2)}{16(1-2\alpha)|\cN_w(i)|n_{i,k}^m}} + \frac{2C_m}{(1-2\alpha)|\cN_w(i)|N_m}.\]
    Similarly, $-\frac{B_{i,k}^m}{(1-2\alpha)|\cN_w(i)|n_{i,k}^m}$ also meets this bound with probability at least $1 - \frac{\delta}{8KV\ln(VT)}$. Therefore, we have
    \begin{equation}
        \Pr\left[\,\left|\frac{B_{i,k}^m}{(1-2\alpha)|\cN_w(i)|n_{i,k}^m}\right| \geq \sqrt{\frac{2\ln(2VT)}{16(1-2\alpha)|\cN_w(i)|n_{i,k}^m}} + \frac{2C_m}{(1-2\alpha)|\cN_w(i)|N_m}\,\right] \leq \frac{1}{2V^2T^2}.\label{eq:partB case2}
    \end{equation}
    We now apply the union bound to combine Eq. \eqref{eq:partA case2} and Eq.\eqref{eq:partB case2}, thereby completing the proof.
\end{proof}
We define an event $\cE$ for epoch $m$ as follows:
\begin{equation*}
    \cE \triangleq \left\{ \forall\ i,\ k,\ m: |r_{i,k}^m - \mu_k| \leq \sqrt{\frac{8\ln(2VT)}{(1-2\alpha)|\cN_w(i)|n_{i,k}^m}} + \frac{2C_m}{(1-2\alpha)|\cN_w(i)|N_m}\right\}.
\end{equation*}
Then we can establish a lower bound on the probability of the event $\cE$ occurring by the following lemma.
\begin{lemma}
\label{lem:event e probability case2}
     The event $\cE$ holds with probability at least $1 - \frac{K\ln(VT)}{VT^2}$.
\end{lemma}
\begin{proof}
    By Lemma~\ref{lem:experiment reward bound case2}, we can get the following inequality for any $i, k$ and $m$:
    \[\Pr\left[\,|r_{i,k}^m - \mu_k| \geq \sqrt{\frac{8\ln(2VT)}{(1-2\alpha)|\cN_w(i)|n_{i,k}^m}} + \frac{2C_m}{(1-2\alpha)|\cN_w(i)|N_m}\,\right] \leq \frac{1}{V^2T^2}.\]
    A union bound over $K$ arms, the $V$ agents, and at most $\ln(VT)$ epochs indicates that the success probability of event $\cE$ is at least $1 - \frac{K\ln(VT)}{VT^2}$.
\end{proof}
Next, we will bound $\Delta_{i,k}^m$. To start, we define the discounted offset rate as
\[\rho_m := \sum_{s=1}^m \frac{C_s}{8^{m-s}(1-2\alpha)v_{\min}^wN_s}.\]
\begin{lemma}
\label{lem:suboptimality-gap bound case2}
    For any fixed $i, k$ and $m$, it follows that
    \[\frac{6}{7}\Delta_k - \frac{3}{4}2^{-m} - 12\rho_m\leq \Delta_{i,k}^{m} \leq \frac{8 \Delta_k}{7} + 2^{-m} + 4\rho_m.\]
\end{lemma}
\begin{proof}
    When the event $\cL_{i,k}^m$ does not occur, we have
    \[\sqrt{\frac{8\ln(2VT)}{(1-2\alpha)|\cN_w(i)|n_{i,k}^m}} = \sqrt{\frac{8\ln(2VT)}{2^9\ln(2VT)(\Delta_{i,k}^{m-1})^{-2}}} = \frac{\Delta_{i,k}^{m-1}}{8}.\]
    When the event $\cL_{i,k}^m$ occurs, we define the agent $i_k' = \argmin_{j \in \cN_w(i)} \tilde n_{j,k}^m$. Then we have
    \[\sqrt{\frac{8\ln(2VT)}{(1-2\alpha)|\cN_w(i)|n_{i,k}^m}} = \sqrt{\frac{8\ln(2VT)}{(1-2\alpha)|\cN_w(i)|\tilde n_{i_k',k}^m}}=\sqrt{\frac{8\ln(2VT)}{2^9\ln(2VT)(\Delta_{i_k',k}^{m-1})^{-2}}} = \frac{\Delta_{i_k',k}^{m-1}}{8},\]
    Since $v_{\min}^w \leq |\cN_w(i)|$, we can get
    \[\frac{2C_m}{(1-2\alpha)v_{\min}^wN_m} - \frac{1}{8}\Delta_{i,k}^{m-1} \leq r_{i,k}^m - \mu_{k} \leq \frac{1}{8}\Delta_{i,k}^{m-1} + \frac{2C_m}{(1-2\alpha)v_{\min}^wN_m}.\]
    or
    \[\frac{2C_m}{(1-2\alpha)v_{\min}^wN_m} - \frac{1}{8}\Delta_{i_k',k}^{m-1} \leq r_{i,k}^m - \mu_{k} \leq \frac{1}{8}\Delta_{i_k',k}^{m-1} + \frac{2C_m}{(1-2\alpha)v_{\min}^wN_m}.\]
    Additionally, let $z = i_*'$ or $i$, given that
    \[r_{i,*}^m \leq \max_{k\in[K]} \left\{\mu_{k} + \frac{1}{8} \Delta_{z,k}^{m-1} - \frac{1}{8} \Delta_{z,k}^{m-1} + \frac{2C_m}{(1-2\alpha)v_{\min}^wN_m}\right\} \leq \mu_{k^*} + \frac{2C_m}{(1-2\alpha)v_{\min}^wN_m},\]
    \[r_{i,*}^m = \max_{k\in[K]} \left\{r_{i,k}^m - \frac{1}{8} \Delta_{z,k}^{m-1}\right\} \geq r_{i,k^*}^m - \frac{1}{8}\Delta_{z,k^*}^{m-1} \geq \mu_{k^*} - \frac{1}{4}\Delta_{z,k^*}^{m-1} - \frac{2C_m}{(1-2\alpha)v_{\min}^wN_m},\]
    it follows that
    \[-\frac{2C_m}{(1-2\alpha)v_{\min}^wN_m} - \frac{\Delta_{z,k^*}^{m-1}}{4} \leq r_{i,*}^m - \mu_{k^*} \leq \frac{2C_m}{(1-2\alpha)v_{\min}^wN_m}.\]
    We now establish the upper bound for $\Delta_{i,k}^m$ using induction on epoch $m$. \\
    For the base case $m = 1$, the statement is trivial as $\Delta_{i,k}^0 = 1$ for all $k \in [K]$. \\
    Assuming the statement is true for $m-1$, we then have
    \begin{equation*}
    \begin{split}
        \Delta_{i,k}^m &= r_{i,*}^m - r_{i,k}^m
        = (r_{i,*}^m - \mu_{k^*}) + (\mu_{k^*} - \mu_k) + (\mu_k - r_{i,k}^m) \\
        &\leq \frac{2C_m}{(1-2\alpha)v_{\min}^wN_m} + \Delta_k + \frac{2C_m}{(1-2\alpha)v_{\min}^wN_m} + \frac{1}{8}\max_{j \in \cN_w(i)}\Delta_{j,k}^{m-1} \\
        &\leq \frac{4C_m}{(1-2\alpha)v_{\min}^wN_m} + \Delta_k + \frac{1}{8}\left(\frac{8 \Delta_k}{7} + 2^{-(m-1)} + 4\rho_{m-1}\right)
        \leq \frac{8 \Delta_k}{7} + 2^{-m} + 4\rho_m,
    \end{split}
    \end{equation*}
    where the second inequality follows from the induction hypothesis. \\
    Next, we establish the lower bound for $\Delta_{i,k}^m$. Specifically, we demonstrate that
    \begin{equation*}
    \begin{split}
        \Delta_{i,k}^m &=  r_{i,*}^m - r_{i,k}^m
        = (r_{i,*}^m - \mu_{k^*}) + (\mu_{k^*} - \mu_k) + (\mu_k - r_{i,k}^m) \\
        &\geq - \frac{2C_m}{(1-2\alpha)v_{\min}^wN_m} - \frac{1}{4}\max_{j \in \cN_w(i)}\Delta_{j,k^*}^{m-1} + \Delta_k - \frac{2C_m}{(1-2\alpha)v_{\min}^wN_m} - \frac{1}{8}\max_{j \in \cN_w(i)}\Delta_{j,k}^{m-1} \\
        &\geq - \frac{4C_m}{(1-2\alpha)v_{\min}^wN_m} + \Delta_k - \frac{1}{8}\left(\frac{8 \Delta_{k}}{7} + 2^{-(m-1)} + 4\rho_{m-1}\right) - \frac{1}{4}\left(\frac{8 \Delta_{k^*}}{7} + 2^{-(m-1)} + 4\rho_{m-1}\right) \\
        &\geq \frac{6}{7}\Delta_k - \frac{3}{4} 2^{-m} -12\rho_m,
    \end{split}
    \end{equation*}
    where the third inequality comes from the upper bound of $\Delta_{i,k}^{m-1}$.  
\end{proof}
Next we will bound the regret and partition the proof into three cases. In each epoch $m$, for any arm $k \neq k_i^m$, we have
\[\tilde n_{i,k}^m = \min \left\{\lambda 2^{2(m-1)}, \frac{16\lambda (\Delta_{i,k}^{m-1})^{-2}}{(1-2\alpha)v_i^w} \right\},\]
and for arm $k_i^m$ we have
\[\tilde n_{i,k_i^m}^m = N_m - \sum_{k \neq k_i^m} \tilde n_{i,k}^m < N_m.\]
\paragraph{Case 1:} $\rho_{m-1} \geq \frac{\Delta_k}{72}$. \\
We define $\cZ^m$ as the set consisting of all arms that satisfy $\rho_{m-1} \geq \frac{\Delta_k}{72}$ at epoch $m$.
\begin{align*}
        \sum_{m=1}^M \sum_{k \in \cZ^m} \rho_{m-1}\tilde n_{i,k}^m
        &\leq \sum_{m=1}^M \rho_{m-1} N_m \\
        &\leq \sum_{m=1}^M \left(\sum_{s=1}^{m-1}\frac{C_s}{8^{m-1-s}(1-2\alpha)v_{\min}^wN_s}\right)N_m \\
        &= 4\sum_{m=1}^M \left(\sum_{s=1}^{m-1}\frac{4^{m-1-s} + 1}{8^{m-1-s}(1-2\alpha)v_{\min}^w}C_s\right) \\
        &= \frac{4}{(1-2\alpha)v_{\min}^w}\sum_{m=1}^M (\sum_{s=1}^{m-1}\left((1/2)^{m-1-s} + (1/8)^{m-1-s}\right)C_s) \\
        &= \frac{4}{(1-2\alpha)v_{\min}^w}\sum_{s=1}^{M-1} C_s \sum_{m=s+1}^M \left((1/2)^{m-1-s} + (1/8)^{m-1-s}\right) \\
        &\leq \frac{4C}{(1-2\alpha)v_{\min}^w}\sum_{j=0}^{\infty} \left((1/2)^{j} + (1/8)^{j}\right)
        \leq \frac{88C}{7(1-2\alpha)v_{\min}^w}.
\end{align*}
Therefore, we obtain
\[\sum_{m=1}^M \sum_{k \in \cZ^m} \tilde n_{i,k}^m \Delta_k \leq 72 \sum_{m=1}^M \sum_{k \in \cZ^m} \rho_{m-1}\tilde n_{i,k}^m \leq \frac{906C}{(1-2\alpha)v_{\min}^w}.\]
For the regret generated when running the algorithm Filter, if $w \leq N_m$, which means that $m \geq \ln(wV/\Delta)$, then we have
\[w \Delta_{i,k}^m \leq 72 w\rho_{m-1} \leq 72\rho_{m-1} N_m \leq \frac{906C}{(1-2\alpha)v_{\min}^w}.\]

\paragraph{Case 2:} $\Delta_k \leq 4 \cdot 2^{-m}$ and $\rho_{m-1} \leq \frac{\Delta_k}{72}$. \\
Since $\Delta_{i,k}^{m-1} = \max \{2^{-(m-1)}, r_{i,*}^{m-1} - r_{i,k}^{m-1}\}$, we have
\[\forall \ i: \quad \Delta_{i,k}^{m-1} \geq 2^{-(m-1)} \geq \frac{\Delta_k}{2}.\]
Therefore, we can get the following inequality for all arms $k \neq k_i^m$:
\[\tilde n_{i,k}^m = \min \left\{\lambda 2^{2(m-1)}, \frac{16\lambda (\Delta_{i,k}^{m-1})^{-2}}{(1-2\alpha)v_i^w} \right\} \leq \frac{16\lambda (\Delta_{i,k}^{m-1})^{-2}}{(1-2\alpha)v_i^w} \leq \frac{64\lambda}{(1-2\alpha)v_i^w\Delta_k^2}.\]
For arm $k_i^m$, since $\Delta_{i,k_i^m}^{m-1} = 2^{-(m-1)}$, we have
\begin{align*}
    \tilde n_{i,k_i^m}^m < N_m = \left\lceil \frac{K\lambda 2^{2(m-1)}}{(1-2\alpha)v_{\min}^w} \right\rceil \leq
    \frac{K\lambda (\Delta_{k_i^m}^{m-1})^{-2}}{(1-2\alpha)v_{\min}^w} + 1 \leq \frac{4K\lambda}{(1-2\alpha)v_{\min}^w\Delta_{k_i^m}^2} + 1 \leq \frac{4K\lambda}{(1-2\alpha)v_{\min}^w\Delta^2} + 1.
\end{align*}
This epoch, which satisfies the given conditions $\Delta_k \leq 4 \cdot 2^{-m}$, is bounded by $\log(1/\Delta)$, which can be considered as a constant.

\paragraph{Case 3:} $\Delta_k > 4 \cdot 2^{-m}$  and $\rho_{m-1} \leq \frac{\Delta_k}{72}$. \\
In this case, by Lemma \ref{lem:suboptimality-gap bound case1} we have
\[\forall \ i: \quad \Delta_{i,k}^{m-1} \geq \frac{6}{7}\Delta_k - \frac{3}{4}2^{-m} -12\rho_m \geq \Delta_k\left(\frac{6}{7} - \frac{1}{4} - \frac{12}{72}\right) \geq 0.5 \Delta_k.\]
In this case, it is impossible for $\Delta_{k_i^m} > 4 \cdot 2^{-m}$ to occur. Since $\Delta_{i,k_i^m}^{m-1} \geq 0.5 \Delta_{k_i^m} > 2^{-(m-1)}$, this does not align with the algorithm's selection criterion $\Delta_{i,k_i^m}^{m-1} = 2^{-(m-1)}$. Therefore, arm $k_i^m$ must be the optimal arm. \\
So we can obtain for all suboptimal arms
\begin{align*}
    \tilde n_{i,k}^m = \min \left\{\lambda 2^{2(m-1)}, \frac{16\lambda (\Delta_{i,k}^{m-1})^{-2}}{(1-2\alpha)v_i^w} \right\} \leq \frac{16\lambda (\Delta_{i,k}^{m-1})^{-2}}{(1-2\alpha)v_i^w} \leq \frac{16\lambda}{0.5^2(1-2\alpha)v_i^w\Delta_k^2}
    \leq \frac{64\lambda}{(1-2\alpha)v_i^w\Delta_k^2}.
\end{align*}
Based on the cases mentioned above, we have the following inequality:
\begin{align*}
        R_i(T) &\leq \sum_{m=1}^M \sum_{\Delta_k > 0} \Delta_{k} \tilde n_{i,k}^m + \sum_{m=1}^M (w - 1)\Delta_{k_i^m} + \frac{KT\ln(VT)}{VT^2} \\ 
        &\leq \sum_{m=1}^M \sum_{\Delta_k > 0} \Delta_{k} \tilde n_{i,k}^m + (w-1)\ln\left(\frac{wV}{\Delta}\right) + \frac{K\ln(VT)}{VT}\\
        &\leq \sum_{m=1}^M \sum_{k \in \cZ^m}\Delta_k\tilde n_{i,k}^m + \sum_{m=1}^M \sum_{\Delta_k > 0, k \not\in \cZ^m}\left(\frac{64\lambda\Delta_k}{(1-2\alpha)v_i^w\Delta_k^2} + \BI(\Delta_k < 4\cdot 2^{-m})\left(\frac{4K\lambda}{(1-2\alpha)v_{\min}^w\Delta^2} + 2\right)\right) \\
        &\quad +(w-1)\ln\left(\frac{wV}{\Delta}\right) \\
        &= O\left(\frac{C}{(1-2\alpha)v_{\min}^w} + \frac{\ln^2(VT)}{(1-2\alpha)v_i^w\Delta_k} + \frac{K\ln(VT)\ln(\Delta^{-1})}{(1-2\alpha)v_{\min}^w\Delta}\right).
\end{align*}
Since the agents only communicate at the end of each epoch, we have
\[\textrm{Cost}(T) = \sum_{i\in [V]}wM = wV\ln(VT).\]

\section{Byzantine Setting}

First, let $\cC$ denote the set of Byzantine agents, and define an event $\cL'$ as follows:
\begin{equation*}
    \cL' \triangleq \left\{ \forall\ i \not\in \cC,\ k,\ m:\quad n_{i,k}^m \leq \tilde n_{j,k}^m \quad \textit{for all} \; j \in \cN_w(i)\right\}.
\end{equation*}
Since for any agent $i \not\in \cC$, we have $|\cC \cap \cN_w(i)| \leq \alpha |\cN_w(i)|$, it follows that after removing from $\cA_{i,k}^m$ all agents $j$  that satisfy $n_{i,k}^m > \tilde n_{j,k}^m$, the set $|\cA_{i,k}^m| \geq |\cN_w(i)| - \alpha |\cN_w(i)|$. Furthermore, when the event $\cL'$ occurs, the event $\cL_{i,k}^m$ will never happen for all $i,k$ and $m$.
\begin{lemma}
\label{lem:event l probability case3}
     The event $\cL'$ holds with probability at least $1 - \frac{K\ln(VT)}{VT^2}$.
\end{lemma}
The proof will be discussed later.

\begin{lemma}
\label{lem:experiment reward bound case3}
    If the event $\cL'$ occurs, for any fixed $i \not\in \cC, k, m$, Algorithm \ref{algs:DeMABAR} satisfies
    \[\Pr\left[\,|r_{i,k}^m - \mu_k| \geq \sqrt{\frac{4\ln(2VT)}{(1-2\alpha)|\cN_w(i)| n_{i,k}^m}}\,\right] \leq \frac{1}{V^2T^2}.\]
\end{lemma}
\begin{proof}
    During each epoch $m$, agent $i$ pulls arm $k$ with probability $p_i^m(k) = \tilde n_{i,k}^m / N_m$. Consider $Y_{i,k}^t$, an indicator variable that determines whether agent $i$ pulls arm $k$. Let $E_m := [T_{m-1} + 1, \ldots, T_m]$ represent the $N_m$ time-steps constituting epoch $m$.
    
    Now, we explain why the impact of Byzantine agents can be completely removed in this case. For agent $i$, let the set of Byzantine agents in $i$'s communication domain be denoted by $\cC_i = \cC \cap \cN_w(i)$. 
    According to Algorithm \ref{alg:filter}, when the event $\cL'$ occurs, any agent $j \in \cN_w(i)$ that satisfies $\tilde n_{j,k}^m < n_{i,k}^m$ must be a Byzantine agent. Moreover, after removing agent $j$ from $\cA_{i,k}^m$ that satisfies $n_{i,k}^m > \tilde n_{j,k}^m$, let $\gamma = |\cN_w(i)| - |\cA_{i,k}^m|$; then there are at most $\alpha |\cN_w(i)| - \gamma$ Byzantine agents in $\cA_{i,k}^m$. Furthermore, let $z = |\cB_{i,k}^m \cap \cC_i|$ denote the number of Byzantine agents that are retained. This implies that there are at least $\frac{|\cA_{i,k}^m| - (1-2\alpha)|\cN_w(i)|}{2} - (\gamma - z) > z$ normal agents who are excluded due to having a smaller $\frac{S_{j,k}^m}{\tilde n_{j,k}^m}$ and at least $z$ normal agents who are excluded due to having a larger $\frac{S_{j,k}^m}{\tilde n_{j,k}^m}$. In this context, for any agent $j$ that satisfies $j \in \cB_{i,k}^m \cap \cC_i$, there exists a pair of distinct normal agents $j^-$ and $j^+$ who are filtered out, such that
    \[\frac{S_{j^-,k}^m}{\tilde n_{j^-,k}^m} \leq \frac{S_{j,k}^m}{\tilde n_{j,k}^m} \leq \frac{S_{j^+,k}^m}{\tilde n_{j^+,k}^m}\]
    Thus, $\frac{S_{j,k}^m}{\tilde n_{j,k}^m}$ can be represented as a convex combination of $\frac{S_{j^-,k}^m}{\tilde n_{j^-,k}^m}$ and $\frac{S_{j^+,k}^m}{\tilde n_{j^+,k}^m}$, as follows:
    \begin{align}
        \frac{S_{j,k}^m}{\tilde n_{j,k}^m} = \theta_j \frac{S_{j^-,k}^m}{\tilde n_{j^-,k}^m} + (1-\theta_j) \frac{S_{j^+,k}^m}{\tilde n_{j^+,k}^m}, \quad \theta_j \in [0, 1].\label{eq:totallossdecompositioncase3}
    \end{align}
    Recalling the definition of $r_{i,k}^m$ and because $|\cB_{i,k}^m| \geq (1-2\alpha)|\cN_w(i)|$:
    \[r_{i,k}^m = \min \left\{\frac{1}{|\cB_{i,k}^m|}\sum_{j \in \cB_{i,k}^m}\frac{S_{j,k}^m}{\tilde n_{j,k}^m}, 1\right\} \leq \frac{1}{(1-2\alpha)|\cN_w(i)|}\sum_{j \in \cB_{i,k}^m}\frac{S_{j,k}^m}{\tilde n_{j,k}^m}.\]
    The inequality we intend to control is then represented as:
    \begin{align*}
        r_{i,k}^m &\leq \frac{1}{(1-2\alpha)|\cN_w(i)|}\sum_{j \in \cB_{i,k}^m}\frac{S_{j,k}^m}{\tilde n_{j,k}^m}
        = \frac{1}{(1-2\alpha)|\cN_w(i)|} \sum_{j \in \cA_{i,k}^m}w_j \frac{S_{j,k}^m}{\tilde n_{j,k}^m} \quad  \left(w_j \in [0,1], \; \sum_{j \in \cA_{i,k}^m}w_j = (1-2\alpha)|\cN_w(i)|\right)\\
        &= \frac{1}{(1-2\alpha)|\cN_w(i)|} \sum_{j \in \cA_{i,k}^m} \sum_{t \in E_m}w_j \frac{Y_{j,k}^t r_{j,k}^t}{\tilde n_{j,k}^m} 
        = \frac{1}{(1-2\alpha)|\cN_w(i)|n_{i,k}^m} \sum_{j \in \cA_{i,k}^m} \sum_{t \in E_m}\frac{w_j Y_{j,k}^t r_{j,k}^tn_{i,k}^m}{\tilde n_{j,k}^m} 
    \end{align*}
    where the first equality holds because we decompose $\frac{S_{j,k}^m}{\tilde n_{j,k}^m}$ by (\ref{eq:totallossdecompositioncase3}).
    To simplify the analysis, we focus on the following component:
    \[A_{i,k}^m = \sum_{j \in \cA_{i,k}^m} \sum_{t \in E_m}\frac{w_j Y_{j,k}^t r_{j,k}^tn_{i,k}^m}{\tilde n_{j,k}^m}.\]
    Notice that $r_{j,k}^t$ is independently drawn from an unknown distribution with mean $\mu_k$, and $Y_{j,k}^t$ is independently drawn from a Bernoulli distribution with mean $q_{j,k}^m := \tilde n_{j,k}^m / N_m$. Since $\tilde n_{j,k}^m \geq n_{i,k}^m$, we have
    \begin{align*}
        \forall \; j,k,m,t:\quad \frac{w_j Y_{j,k}^t r_{j,k}^tn_{i,k}^m}{\tilde n_{j,k}^m} \leq 1.
    \end{align*}
    Furthermore, we can obtain
    \begin{align*}
        \BE[A_{i,k}^m] 
        = \sum_{j \in \cA_{i,k}^m} \sum_{t \in E_m}\frac{w_j Y_{j,k}^t r_{j,k}^tn_{i,k}^m}{\tilde n_{j,k}^m} 
        = \sum_{j \in \cN_w(i)} w_j n_{i,k}^m \mu_k 
        = (1-2\alpha) |\cN_w(i)| n_{i,k}^m \mu_k.
    \end{align*}    
    Therefore, by utilizing the Chernoff-Hoeffding inequality (Theorem 1.1 in \cite{dubhashi2009concentration}), we derive the following result:
    \[\Pr\left[\,\left|A_{i,k}^m - (1-2\alpha) |\cN_w(i)| n_{i,k}^m \mu_k\right| \geq \sqrt{3(1-2\alpha) |\cN_w(i)| n_{i,k}^m \mu_k\ln(4V^2T^2)}\,\right] \leq \frac{1}{2V^2T^2}.\]
    Through simple calculations, we can get
    \begin{equation}
        \Pr\left[\,\left|\frac{A_{i,k}^m}{(1-2\alpha)|\cN_w(i)| n_{i,k}^m} - \mu_k\right| \geq \sqrt{\frac{6\ln(2VT)}{(1-2\alpha)|\cN_w(i)| n_{i,k}^m}}\,\right] \leq \frac{1}{2V^2T^2}.\label{eq:partA case3}
    \end{equation}
    The proof is complete.
\end{proof}
We define an event $\cE$ for epoch $m$ as follows:
\begin{equation*}
    \cE \triangleq \left\{ \forall\ i,\ k ,\ m: |r_{i,k}^m - \mu_k| \leq \sqrt{\frac{8\ln(2VT)}{(1-2\alpha)|\cN_w(i)|n_{i,k}^m}}\right\}.
\end{equation*}
Then we can establish a lower bound on the probability of the event $\cE$ occurring by the following lemma.
\begin{lemma}
\label{lem:event e probability case3}
     The event $\cE$ holds with probability at least $1 - \frac{K\ln(VT)}{VT^2}$.
\end{lemma}
\begin{proof}
    By Lemma~\ref{lem:experiment reward bound case3}, we can get the following inequality for any $i, k$ and $m$:
    \[\Pr\left[\,|r_{i,k}^m - \mu_k| \geq \sqrt{\frac{8\ln(2VT)}{(1-2\alpha)|\cN_w(i)|n_{i,k}^m}}\,\right] \leq \frac{1}{V^2T^2}.\]
    A union bound over the $K$ arms, $V$ agents, and at most $\ln(VT)$ epochs indicates that the success probability of event $\cE$ is at least $1 - \frac{K\ln(VT)}{VT^2}$.
\end{proof}
Our discussion below will be based on the occurrence of event $\cE$.
\begin{lemma}
\label{lem:suboptimality-gap bound case3}
    For any fixed $i, k$ and $m$, it follows that
    \[\frac{6}{7}\Delta_k - \frac{3}{4}2^{-m} \leq \Delta_{i,k}^{m} \leq \frac{8 \Delta_k}{7} + 2^{-m}.\]
\end{lemma}
\begin{proof}
    First, we have
    \[\sqrt{\frac{8\ln(2VT)}{(1-2\alpha)|\cN_w(i)|n_{i,k}^m}} = \sqrt{\frac{8\ln(2VT)}{2^9\ln(2VT)(\Delta_{i,k}^{m-1})^{-2}}} = \frac{\Delta_{i,k}^{m-1}}{8},\]
    Therefore, we can get
    \[- \frac{1}{8}\Delta_{i,k}^{m-1} \leq r_{i,k}^m - \mu_{k} \leq \frac{1}{8}\Delta_{i,k}^{m-1}.\]
    Additionally, given that
    \[r_{i,*}^m \leq \max_{k\in[K]} \left\{\mu_{k} + \frac{1}{8} \Delta_{i,k}^{m-1} - \frac{1}{8} \Delta_{i,k}^{m-1}\right\} \leq \mu_{k^*},\]
    \[r_{i,*}^m = \max_{k\in[K]} \left\{r_{i,k}^m - \frac{1}{8} \Delta_{i,k}^{m-1}\right\} \geq r_{i,k^*}^m - \frac{1}{8}\Delta_{i,k^*}^{m-1} \geq \mu_{k^*} - \frac{1}{4}\Delta_{i,k^*}^{m-1},\]
    it follows that
    \[- \frac{\Delta_{i,k^*}^{m-1}}{4} \leq r_{i,*}^m - \mu_{k^*} \leq 0.\]
    We now establish the upper bound for $\Delta_{i,k}^m$ using induction on epoch $m$. \\
    For the base case $m = 1$, the statement is trivial as $\Delta_{i,k}^0 = 1$ for all $k \in [K]$. \\
    Assuming the statement is true for $m-1$, we then have
    \begin{equation*}
    \begin{split}
        \Delta_{i,k}^m &= r_{i,*}^m - r_{i,k}^m
        = (r_{i,*}^m - \mu_{k^*}) + (\mu_{k^*} - \mu_k) + (\mu_k - r_{i,k}^m) \\
        &\leq \Delta_k + \frac{1}{8}\Delta_{i,k}^{m-1}
        \leq \Delta_k + \frac{1}{8}\left(\frac{8 \Delta_k}{7} + 2^{-(m-1)}\right)
        \leq \frac{8 \Delta_k}{7} + 2^{-m},
    \end{split}
    \end{equation*}
    Where the second inequality follows from the induction hypothesis. \\
    Next, we establish the lower bound for $\Delta_{i,k}^m$. Specifically, we demonstrate that
    \begin{equation*}
    \begin{split}
        \Delta_{i,k}^m &=  r_{i,*}^m - r_{i,k}^m
        = (r_{i,*}^m - \mu_{k^*}) + (\mu_{k^*} - \mu_k) + (\mu_k - r_{i,k}^m) \\
        &\geq  - \frac{1}{4}\Delta_{i,k^*}^{m-1} + \Delta_k - \frac{1}{8}\Delta_{i,k}^{m-1}
        \geq \Delta_k - \frac{1}{8}\left(\frac{8 \Delta_{k}}{7} + 2^{-(m-1)}\right) - \frac{1}{4}\left(\frac{8 \Delta_{k^*}}{7} + 2^{-(m-1)}\right)
        \geq \frac{6}{7}\Delta_k - \frac{3}{4} 2^{-m}.
    \end{split}
    \end{equation*}
    where the third inequality comes from the upper bound of $\Delta_{i,k}^{m-1}$.
\end{proof}

\begin{lemma}
\label{lem:estimated gap ratio case3}
    For any fixed $k$, $m$, and two agents $i, j$, it follows that
    \[\frac{\Delta_{i,k}^{m-1}}{\Delta_{j,k}^{m-1}} \in \left[\frac{1}{4}, 4\right].\]
\end{lemma}
\begin{proof}
    Since $\Delta_{i,k}^m = \max_k\{2^{-m}, r_{i,*}^m - r_{i,k}^m\} \geq 2^{-m}$ for all $i \in [V]$, we have $\Delta_{i,k}^m \geq 2^{-m}$ and $\Delta_{j,k}^m \geq 2^{-m}$. By Lemma \ref{lem:suboptimality-gap bound case3}, when $\Delta_k \leq \frac{49}{24}2^{-m}$, then we have
    \[\frac{6}{7}\Delta_k - \frac{3}{4}2^{-m} \leq 2^{-m}.\]
    Hence, we can get
    \[\frac{\Delta_{i,k}^m}{\Delta_{j,k}^m}  \leq \frac{\frac{8}{7}\Delta_k + 2^{-m}}{2^{-m}} = 1 + \frac{8\Delta_k}{7\cdot 2^{-m}} < 4.\]
    When $\Delta_k \geq \frac{49}{24}2^{-m}$, then we have
    \[\frac{\Delta_{i,k}^m}{\Delta_{j,k}^m} \leq \frac{\frac{8}{7}\Delta_k + 2^{-m}}{\frac{6}{7}\Delta_k - \frac{3}{4}2^{-m}} = \frac{4}{3} + \frac{2\cdot 2^{-m}}{\frac{6}{7}\Delta_k - \frac{3}{4}2^{-m}} < 4.\]
    The inequality reaches its maximum value when $\Delta_k = \frac{49}{24}2^{-m}$.
    Because $i$ and $j$ are equivalent, the proof can be completed by swapping their positions.
\end{proof}

Since for any agent $j \in \cN_w(i)$, we have
\[\tilde n_{j,k}^m = \min \left\{\lambda 2^{2(m-1)}, \frac{16\lambda (\Delta_{j,k}^{m-1})^{-2}}{(1-2\alpha)v_j} \right\} \geq \min \left\{\lambda 2^{2(m-1)}, \frac{16\lambda (\Delta_{i,k}^{m-1})^{-2} (\Delta_{j,k}^{m-1}/ \Delta_{i,k}^{m-1})^{-2}}{(1-2\alpha)|\cN_w(i)|} \right\} \geq n_{i,k}^m.\]
So we can say that when event $\cE$ occurs, event $\cL$ must occur. Lemma~\ref{lem:event l probability case3} is complete. \\
Next we will bound the regret and partition the proof into two cases. In each epoch $m$, for any arm $k \neq k_i^m$, we have
\[\tilde n_{i,k}^m = \min \left\{\lambda 2^{2(m-1)}, \frac{16\lambda (\Delta_{i,k}^{m-1})^{-2}}{(1-2\alpha)v_i} \right\},\]
and for arm $k_i^m$ we have
\[\tilde n_{i,k_i^m}^m = N_m - \sum_{k \neq k_i^m} \tilde n_{i,k}^m < N_m.\]
\paragraph{Case 1:} $\Delta_k \leq 3 \cdot 2^{-m}$. \\
Since $\Delta_{i,k}^{m-1} = \max \{2^{-(m-1)}, r_{i,*}^{m-1} - r_{i,k}^{m-1}\}$, we have
\[\forall \ i: \quad \Delta_{i,k}^{m-1} \geq 2^{-(m-1)} \geq \frac{2\Delta_k}{3}.\]
Therefore, we can get the following inequality for all arms $k \neq k_i^m$:
\[\tilde n_{i,k}^m = \min \left\{\lambda 2^{2(m-1)}, \frac{16\lambda (\Delta_{i,k}^{m-1})^{-2}}{(1-2\alpha)v_i} \right\} \leq \frac{16\lambda (\Delta_{i,k}^{m-1})^{-2}}{(1-2\alpha)v_i} \leq \frac{36\lambda}{(1-2\alpha)v_i\Delta_k^2}.\]
For arm $k_i^m$, since $\Delta_{i,k_i^m}^{m-1} = 2^{-(m-1)}$, we have
\begin{align*}
    \tilde n_{i,k_i^m}^m < N_m = \left\lceil \frac{K\lambda 2^{2(m-1)}}{(1-2\alpha)v_{\min}} \right\rceil \leq
    \frac{K\lambda (\Delta_{k_i^m}^{m-1})^{-2}}{(1-2\alpha)v_{\min}} + 1 \leq \frac{9K\lambda}{4(1-2\alpha)v_{\min}\Delta_{k_i^m}^2} + 1 \leq \frac{9K\lambda}{4(1-2\alpha)v_{\min}\Delta^2} + 1.
\end{align*}
This epoch, which satisfies the given condition $\Delta_k \leq 3 \cdot 2^{-m}$, is bounded by $\log(1/\Delta)$, which can be considered as a constant.

\paragraph{Case 2:} $\Delta_k > 3 \cdot 2^{-m}$. \\
In this case, by Lemma \ref{lem:suboptimality-gap bound case3} we have
\[\forall \ i: \quad \Delta_{i,k}^{m-1} \geq \frac{6}{7}\Delta_k - \frac{3}{4}2^{-m} \geq \Delta_k\left(\frac{6}{7} - \frac{1}{4}\right) \geq 0.61 \Delta_k.\]
In this case, it is impossible for $\Delta_{k_i^m} > 3 \cdot 2^{-m}$ to occur. Since $\Delta_{i,k_i^m}^{m-1} \geq 0.61 \Delta_{k_i^m} > 2^{-(m-1)}$, this does not align with the algorithm's selection criterion $\Delta_{i,k_i^m}^{m-1} = 2^{-(m-1)}$. Therefore, arm $k_i^m$ must be the optimal arm. \\
So we can obtain the following bound for all suboptimal arms:
\begin{align*}
    \tilde n_{i,k}^m = \min \left\{\lambda 2^{2(m-1)}, \frac{16\lambda (\Delta_{i,k}^{m-1})^{-2}}{(1-2\alpha)v_i} \right\} \leq \frac{16\lambda (\Delta_{i,k}^{m-1})^{-2}}{(1-2\alpha)v_i} \leq \frac{16\lambda}{0.61^2(1-2\alpha)v_i\Delta_k^2}
    \leq \frac{43\lambda}{(1-2\alpha)v_i\Delta_k^2}.
\end{align*}
Based on the cases mentioned above, we have the following inequality:
\begin{align*}
        R_i(T) &\leq \sum_{m=1}^M \sum_{\Delta_k > 0} \Delta_{k} \tilde n_{i,k}^m + \sum_{m=1}^M (w - 1)\Delta_{k_i^m} + \frac{KT\ln(VT)}{VT^2} \\
        &\leq \sum_{m=1}^M \sum_{\Delta_k > 0} \Delta_{k} \tilde n_{i,k}^m + \frac{K\ln(VT)}{VT}\\
        &\leq \sum_{m=1}^M \sum_{\Delta_k > 0} \left(\Delta_k \frac{43\lambda}{(1-2\alpha)v_i\Delta_k^2} + \BI(\Delta_k < 3\cdot 2^{-m})\left(\frac{9K\lambda}{4(1-2\alpha)v_{\min}\Delta^2} + 2\right)\right) \\
        &= O\left(\frac{\ln^2(VT)}{(1-2\alpha)v_i\Delta_k} + \frac{K\ln(VT)\ln(\Delta^{-1})}{(1-2\alpha)v_{\min}\Delta}\right).
\end{align*}
Since the agents only communicate at the end of each epoch, we have
\[\textrm{Cost}(T) = \sum_{i\in [V]}wM = V\ln(VT).\]

\section{Experimental Details}
In this section, we introduce the implementation details of the experiments. Unless otherwise stated, we set the fraction $\alpha = \frac{1}{3}$ for DeMABAR.
\subsection{Multi-Agent Bandits with Adversarial Corruption}
We set $k_i = 1.5$ for Decentralized Robust UCB \cite{zhu2023byzantine}. Other parameters are the same as the setting in \cite{zhu2023byzantine}.
For IND-FTRL \cite{zimmert2021tsallis}, we use the importance-weighted unbiased loss estimators to construct the algorithm. 
For IND-BARBAR \cite{gupta2019better}, DRAA \cite{ghaffari2024multi}, MA-BARBAT and DeMABAR, we choose the parameter $\lambda = 5 \ln\left(4V^2T\right)$. Other parameters follow the original settings in the algorithms.

\subsection{Byzantine Decentralized Multi-Agent Bandits}

Compared to the adversarial damage setting, we did not change any algorithm-related parameters.

\end{document}